\documentclass[11pt,draftclsnofoot,onecolumn]{IEEEtran}
\usepackage[utf8]{inputenc} 
\usepackage[T1]{fontenc}    
\usepackage{hyperref}       
\usepackage{url}            
\usepackage{booktabs}       
\usepackage{amsfonts}       
\usepackage{nicefrac}       
\usepackage{microtype}      
\usepackage{amsmath}
\usepackage{amssymb}
\usepackage{latexsym}
\usepackage{yhmath}
\usepackage{CJK}
\usepackage{cite}

\usepackage{color, soul}
\definecolor{shadecolor}{rgb}{1,0.94,0}
\usepackage{framed}

\usepackage{graphicx}
\usepackage{epstopdf}
\usepackage{epsfig}
\usepackage{subfigure} %

\usepackage{verbatim}  

\usepackage{mathrsfs}	
\usepackage{algorithm} 
\usepackage{algorithmic} 
\usepackage{textcomp}
\usepackage{multirow}
\usepackage{lettrine}

\usepackage{mathrsfs}

\usepackage{algorithm}
\usepackage{algorithmic}
\usepackage{extarrows}
\usepackage{colortbl}
\definecolor{mygray}{gray}{.9}

\usepackage{amsthm} 

\newtheorem{cor}{Corollary}
\newtheorem{lem}{Lemma}
\newtheorem{prop}{Proposition}
\newtheorem{defn}{Definition}

{ \bgroup
  \addtolength\abovedisplayshortskip{#1}
  \addtolength\abovedisplayskip{#1}
  \addtolength\belowdisplayshortskip{#1}
  \addtolength\belowdisplayskip{#1}}
{\egroup\ignorespacesafterend}

\usepackage{subeqnarray}%
\usepackage{cases}%
\usepackage{mathtools} 
\usepackage{bm}

\usepackage{stfloats}

\begin{document}

\title{From MIM-Based GAN to Anomaly Detection: Event Probability Influence on Generative Adversarial Networks}

\author{Rui~She,
        Pingyi~Fan
\IEEEcompsocitemizethanks{
\IEEEcompsocthanksitem R.~She and P. Fan are with Beijing National Research Center for Information Science and Technology and the Department of Electronic Engineering, Tsinghua University, Beijing, 100084, China.
E-mail: sher15@mails.tsinghua.edu.cn; fpy@tsinghua.edu.cn.
}}


\maketitle

\begin{abstract}
In order to introduce deep learning technologies into anomaly detection, Generative Adversarial Networks (GANs) are considered as important roles in the algorithm design and realistic applications.
In terms of GANs, event probability reflected in the objective function, has an impact on the event generation which plays a crucial part in GAN-based anomaly detection.
The information metric, e.g. Kullback-Leibler divergence in the original GAN, makes the objective function have different sensitivity on different event probability, which provides an opportunity to refine GAN-based anomaly detection by influencing data generation.
In this paper, we introduce the exponential information metric into the GAN, referred to as MIM-based GAN, whose superior characteristics on data generation are discussed in theory.
Furthermore, we propose an anomaly detection method with MIM-based GAN, as well as explain its principle for the unsupervised learning case from the viewpoint of probability event generation.
Since this method is promising to detect anomalies in Internet of Things (IoT), such as environmental, medical and biochemical outliers, we make use of several datasets from the online ODDS repository to evaluate its performance and compare it with other methods.
\end{abstract}


\IEEEpeerreviewmaketitle

\section{Introduction}\label{sec:introduction}

\subsection{Background}
Towards the smart and wireless-Internet era, it is favored and required to introduce Artificial Intelligence (AI) technologies into the real-world applications of Internet of Things (IoT) \cite{An-intelligent,DeepEDN,Recent-GAN-survey}.
As a widely used AI technology, Generative Adversarial Networks (GANs) \cite{Generative-Adversarial-Nets} are feasible for complex or high-dimensional data processing that is commonly suffered in IoT \cite{Selective-unsupervised,Applying-cross-modality}.
In practice, there are a great deal of applications based on GANs \cite{ESR-GAN,Generating-Videos,Imitating-driver-behavior,A-novel-semi-supervised,The-secure-steganography}, such as image reconstruction,
behavior imitation, pedestrian reidentification and secure steganography, especially anomaly detection.

In fact, anomaly detection, as a research direction with many realistic uses, has attracted a lot of attentions 
in numerous scenarios, such as fraud identification \cite{fraud-detection}, manufacturing quality control \cite{Big-data- driven}, emergency alarm \cite{ATD} and medical screening \cite{Probabilistic-mismatch}, especially in IoT or Industrial IoT (IIoT).
In terms of IoT, there are a huge number of sensors generating large volumes of multidimensional data, where data anomalies caused by system errors or malicious attacks do harm to the application  effectiveness or even make the system break down \cite{Big_Data_IoT,Two_tier,6G_Enabled}.
As a crucial challenge in IoT, anomaly detection has been widely investigated in system security and monitoring management fields, including Intrusion Detection Systems (IDSs) \cite{Big_Data_IoT,Two_tier,Cyber_Physical,Distributed_Intrusion,GAN_Auto_Encoder},
fraud monitoring \cite{Make-the-rocket,CHRIST},
data-based decision \cite{Monotone,Wireless_Commun,feature-based}
and so on.
Conceptually, anomaly detection refers to digging out the abnormal data that is inconsistent with the general data in the viewpoint of given rules or features, which is also regarded as a binary-classification issue for normal data and abnormal data.

As for usual anomaly detection methods, it is popular to use a model to describe the characteristics of normal data and then find the abnormal data having large deviation from this model \cite{Isolation-forest,Isolation-based,Estimating-the-support,High-dimensional-and-large-scale}, which corresponds to the core idea of density-based methods \cite{exemplar-based-GMM,Robustmodel-based,Deep-autoencoding-Gaussian}, distance-based methods \cite{LOF,Fast-memory,A-near-linear-time,Outlier-detection}, reconstruction-based methods \cite{Direct-robust-matrix,Fast-matrix-factorization,robust-deep-autoencoders,autoencoder-ensembles} and mixture methods \cite{K-means-Clustering-PCA,Dimensionality-Reduction}.
However, in these methods, not only different assumptions and prior information are needed to select appropriate parameters of models, but high-dimensional data also damages the detection efficiency and effectiveness.

With the development of sensors for data acquisition, multi-sensor fusion data is usually obtained which is difficult to handle and analyze by using conventional methods.
In this case, learning-based neural networks technologies \cite{A-deep-one-class,Anomaly-detection-of-time}, such as Convolutional Neural Networks (CNNs), Recurrent Neural Networks (RNNs), Variational Auto-Encoder (VAE) networks, are considered to bring gains or novel results into the anomaly detection.
However, these existing networks mainly focus on handling the data with particular characteristics, such as spatial data or temporal data, which implies it is necessary to alter the networks according to specific datasets.
While, by resorting to the combination structure for neural networks, GANs are introduced into anomaly detection \cite{Unsupervised-anomaly,Adversarially-learned,Adversarial-feature,GANomaly,f-AnoGAN}, which not only provide a game model to deal with multidimensional data distributions but also have advantages on incorporating different networks to process different datasets.

As for GAN-based anomaly detection, there are almost weakly supervised or semi-supervised methods with partial normal data labels, as well as the unsupervised methods with prior information \cite{Telemetry-data-based,Anomaly-monitoring,Weakly-supervised,Semisupervised-spectral, Unsupervised-pixel-wise}, such as spatial and temporal characteristics or other preprocessing based on data attributes.
Few totally unsupervised methods are investigated for unordered high-dimensional data.
Moreover, few works theoretically explain the intrinsic principle of GAN-based anomaly detection, from the perspective of event probability.

In this work, we take advantage of a new kind of GAN with the exponential information metric to detect anomalies hidden in unordered discrete data, without any supervised training or prior information.
Especially, event probability, as a fundamental factor on the data generation of GANs, is analyzed from the viewpoint of its influence on the objective function.
Based on this, we provides a theoretical explanation why the proposed GAN-based method is workable for entirely unsupervised anomaly detection.

\subsection{Related works}
\subsubsection{Anomaly detection}\ \par
As far as anomaly detection is concerned, it is a straightforward method to establish a model for normal data profiles and pick out the anomalous data not fitting the model well.
The Gaussian Mixture Model (GMM) \cite{exemplar-based-GMM} is applied to estimate the density of the given data to detect anomalies, whose parameters are evaluated by use of the globally optimal Expectation-Maximization (EM) algorithm \cite{Robustmodel-based} or an estimation neural network \cite{Deep-autoencoding-Gaussian}.
Instead of using the density of data, distance-based methods also provide a way to detect anomalies, which measure the rarity of each object, such as
Local Outlier Factor (LOF) \cite{LOF,Fast-memory} and Angle-based Outlier Factor (AOF)\cite{A-near-linear-time}.
To handle multivariate data, Minimum Covariance Determinant (MCD) estimator \cite{Outlier-detection} is used to robustly estimate the corresponding location and scatter based on Mahalanobis squared distance.
While, different from the popular methods focusing on normal data profiling, Isolation Forest \cite{Isolation-forest,Isolation-based} makes use of anomaly isolation to distinguish abnormal data from normal data.

To achieve anomaly detection with high-dimensional data, One-Class Support Vector Machine (OC-SVM) \cite{Estimating-the-support,High-dimensional-and-large-scale} is used to obtain a hyperplane to classify the original data into the majority set or minority set, while it also needs appropriate parameters of model before the testing process.
Besides, the reconstruction-based methods, such as Matrix Factorization (MF) \cite{Direct-robust-matrix,Fast-matrix-factorization} as well as deep Auto-encoders \cite{robust-deep-autoencoders,autoencoder-ensembles}, are considered to detect anomalies with high reconstruction error using the compression and decompression idea.
Moreover, it is available to pick out anomalies hidden in high-dimensional data by mixing several techniques, such as K-Means algorithm and Principal Component Analysis (PCA) \cite{K-means-Clustering-PCA,Dimensionality-Reduction}, which combines dimensionality reduction and clustering methods to achieve anomaly detection.

Using the game framework based on neural networks, GANs are introduced into anomaly detection to bring many benefits including high-dimensional data processing and incorporating other complex neural networks.
AnoGAN \cite{Unsupervised-anomaly} is proposed to discover unhealthy cases from unseen medical images by using deep CNNs as the generator and discriminator to score the testing samples.
Then, to further stabilize training and improve detection performance, Adversarially Learned Anomaly Detection (ALAD) \cite{Adversarially-learned} is designed by resorting to the structure of Bidirectional GANs \cite{Adversarial-feature} and encoder networks, as well as, GANomaly \cite{GANomaly} is proposed using the generator with encoder-decoder-encoder networks.
Besides, f-AnoGAN \cite{f-AnoGAN} is presented to accelerate training by exploiting a related encoder training procedure and a kind of GAN which takes advantage of Wasserstein distance as the information metric in the objective function.

In many scenarios of anomaly detection, the architectures of GANs and objective functions based on information metrics are adjusted to fit the specific data with different characteristics.
For instance,
CNNs and Long-short Term Memory (LSTM) are used in generators and discriminators to handle the data with spatial and temporal features \cite{Telemetry-data-based,Anomaly-monitoring}.
As well, Auto-Encoder (AE) networks are also introduced into adversarial networks as the reconstruction model to detect anomalies for Hyper-Spectral Images (HSIs) \cite{Weakly-supervised,Semisupervised-spectral, Unsupervised-pixel-wise}.

In addition, anomaly detection is commonly used in IoT applications.
For example, a cluster-based data analysis framework is designed based on recursive principal component analysis to achieve outlier detection and sensor data aggregation in IoT systems \cite{RPCA}.
As well, a fog-empowered anomaly detection method is presented depending on an efficient hyperellipsoidal clustering algorithm, which is suitable for the fog computing architecture \cite{Fog_empowered}.
An anomaly detection algorithm is studied based on multidimensional data processing, which is also with an autoregressive exogenous model for detecting sensor-level anomalies and a Cumulative Coefficient of Value (CCoV) measure for high-value sensing device identification \cite{6G_Enabled}.
A mixture method based on big data analytics is presented for the Network-based IDS (NIDS) of IoT networks, which incorporates the event strength function, Dynamic Pareto Set and multicriteria temporal graphs \cite{Big_Data_IoT}.

Besides, as a kind of popular deep learning method for multidimensional data processing, GANs are also introduced into the anomaly detection for IoT systems, especially for IDS and monitoring applications.
Specifically, a distributed GAN-based IDS method is studied to detect internal and external attacks for IoT without relying on the central control unit too much \cite{Distributed_Intrusion}.
An hierarchical anomaly detection approach is designed with the cooperation of GAN and AE networks, which addresses the issues of the lack of sufficiently-large amount of IoT traffic data and privacy protection \cite{GAN_Auto_Encoder}.
An integrated generative model for IIoT is investigated based on GANs and bi-directional LSTM with attention mechanism to dig out multidimensional industrial time series anomalies \cite{Integrated_GAN}.
As well, an edge-based detection method adopts encoder networks and original GAN to construct the Stepwise GAN (StepGAN) \cite{Make-the-rocket}, which is used to detect anomalies of liquid rocket engine at the edge of IoT.

\subsubsection{Models of GANs}\ \par
From the perspective of objective functions, it is common to investigate the GANs with different information metrics.
As far as the original GAN \cite{Generative-Adversarial-Nets} is concerned, there exists Jensen-Shannon divergence in the objective function, which consists of two symmetric Kullback-Leibler Divergences (KLDs).
Similar to the original GAN, a series of classical GANs, such as InfoGAN \cite{InfoGAN} and Variational Auto-Encoder GAN (VAEGAN) \cite{variational-autoencoder-GAN},
adopt KLD to measure the hidden information distance between real data and generative data.
However, the objective function based on KLD has its disadvantages on effectiveness and efficiency of learning process.
To make up for this, there are some following literatures replacing the KLD with other information distances.

To avoid the vanishing gradients phenomenon, Pearson $\chi^2$ divergence is introduced into the objective function, whose corresponding GAN is named Least Squares Generative Adversarial Networks (LSGAN) \cite{Least-squares-generative-adversarial-networks}.
Another information metric, Earth Mover Distance (EMD) is also adopted in Wasserstein Generative Adversarial Networks (WGAN) to achieve a better training trade-off between the discriminator and the generator \cite{Wasserstein-GAN}.
To enable the stable training without hyper-parameter tuning, an improved WGAN with gradient penalty (named WGAN-GP) is designed by enforcing the Lipschitz constraint \cite{Improved-Training-of-Wasserstein-GANs}.

In fact, based on the core idea of original GAN, there are a series of works to ameliorate GANs from the perspective of architectures, such as Conditional GAN (CGAN) \cite{cGAN} and Deep Convolutional GAN (DCGAN) \cite{DCGAN}.
Based on the architecture modifications in the Bayesian GAN \cite{Bayesian-GAN} and two-discriminators GAN \cite{Collapse-GAN}, the issue of mode collapse in the original GAN is overcome, though training techniques are also available to handle this issue \cite{Improved-training}.
Moreover, there are other GANs with different architectures designed to improve the generation performance for special applications.
For instance, CycleGAN \cite{CycleGAN} makes use of two mapping functions with cycle consistency losses to refine image-to-image translation.
To enable high generation quality for images,
Progressively-Growing GAN \cite{PGGAN} and
Stacked GAN \cite{StackGAN} are designed based on multi-scale architectures for both discriminator and generator.
While, to solve the issue about controlling the image synthesis,
StyleGAN \cite{StyleGAN} is proposed by use of a new architecture with Adaptive Instance Normalization (AdaIN).
As a high-fidelity image generation model, BigGAN \cite{BigGAN} can train large-scale neural networks, which include many components such as Self-Attention and Spectral Normalization.

In our work, an improved GAN is designed from the viewpoint of information metric in the objective function, which is independent from the research route based on complicated architectures.
More importantly, information metrics not only make differences on the training process performance of GANs by altering the corresponding mathematical characteristics of objective functions, but also reflect the event probability influence on data generation efficiency in some degree.
This provides a potential advantage for the anomaly detection based on GANs.

\subsection{Contributions and organization}
We make a summary for the contributions and organization as follows:

\begin{itemize}
\item
By drawing the main thought of Message Importance Measure (MIM) into the objective function, a novel GAN model called MIM-based GAN is introduced, whose special characteristics are discussed with respect to data generation, especially for small probability events.
\item
By resorting to MIM-based GAN, an unsupervised anomaly detection method is proposed, whose principle is also analyzed from the viewpoint of probability event generation.
\item
Several real datasets are used to evaluate the performance of unsupervised anomaly detection for the method with MIM-based GAN and compare it with other methods.
\end{itemize}

The rest part is organized as follows.
In Section II, MIM-based GAN is presented based on the idea of MIM and its several characteristics are also investigated in theory.
In Section III, an unsupervised anomaly detection method with MIM-based GAN is proposed and its principle is explained in theory.
Section IV take experiments to evaluate the data generation performance for MIM-based GAN, as well as
provide the comparison for the method with MIM-based GAN and several other methods in terms of unsupervised anomaly detection.
In the end, we make a conclusion for this work in the Section V.

\section{Model of MIM-based GAN}
\subsection{A general form of two-player game in GANs}
Considering the core idea of GANs, both optimal discriminator and generator networks are trained to generate the data which follows the distribution approximating to the real.
From the conventional perspective, the framework of GANs is designed by use of the two-player game between the two networks, in which the objective functions play important roles.

In general, the two-player game for the objective function optimization can be described as
\begin{equation}\label{eq.GAN_general_optimizaition}
\begin{aligned}
    \min_{G} \max_{D} L(D,G),
\end{aligned}
\end{equation}
as well as the objective function denoted by $L(D,G)$ is defined as
\begin{equation}\label{eq.GAN_general}
\begin{aligned}
    L(D,G)
    & = \mathbb{E}_{{\bm x}\sim \mathbb{P}}[f(D({\bm x}))]
    +\mathbb{E}_{{\bm z}\sim \mathbb{P}_{z}}[g(D(G({\bm z})))] \\
    & = \mathbb{E}_{{\bm x}\sim \mathbb{P}}[f(D({\bm x}))]
    +\mathbb{E}_{{\bm x}\sim \mathbb{P}_{g_{\theta}}}[g(D({\bm x}))] \text{,}
\end{aligned}
\end{equation}
where $f(\cdot)$ and $g(\cdot)$ denote two functions, $D$ denotes the discriminator network, $G$ denotes the generator network, ${\bm x}$ and ${\bm z}$ are the inputs for $D$ and $G$ respectively, $\mathbb{P}$, $\mathbb{P}_{g_{\theta}}$ and $\mathbb{P}_{z}$ denote the corresponding distributions for the real data, generative data and input data of generator, respectively.

In particular, there are several classical GANs in Table \ref{table.summary_GANs}, which are
summarized in the viewpoint of the form of objective function.
In this regard, the essential idea is that information metrics are considered as theoretical units for the specific objective functions of GANs.
\begin{table}[!htb]
\renewcommand{\arraystretch}{1.0}
\centering
\caption{{The GANs summarized according to the form of objective function}}
\label{table.summary_GANs}
\newcommand{\tabincell}[2]{\begin{tabular}{@{}#1@{}}#2\end{tabular}}
\resizebox{\textwidth}{30mm}{
\begin{tabular}{|l|l|l|}
\bottomrule
\rowcolor{mygray}
\textbf{Type of GANs} & \textbf{Objective Function Optimization} &\textbf{Corresponding Functions in the General Form}\\
\hline
\tabincell{l}{Original GAN\\ \cite{Generative-Adversarial-Nets}\\ }
&
\tabincell{l}{
$
\begin{aligned}
&\min\limits_{G} \max\limits_{D} \big\{ \mathbb{E}_{{\bm x}\sim \mathbb{P}}[\ln D({\bm x})]
+\mathbb{E}_{{\bm z}\sim \mathbb{P}_{z}}[\ln (1-D(G({\bm z})))] \big\}\\
&\Leftrightarrow\\
&\max\limits_{G} \min\limits_{D}
\Big\{ \mathbb{E}_{{\bm x}\sim \mathbb{P}}\Big[\ln \big(\frac{1}{D({\bm x})}\big)\Big]
+\mathbb{E}_{{\bm z}\sim \mathbb{P}_{z}}\Big[\ln \big(\frac{1}{1-D(G({\bm z}))}\big)\Big] \Big\}
\end{aligned}
$\\
}
&
\tabincell{l}{
$\bullet$ $f(\cdot)$ and $g(\cdot)$ in Eq. (\ref{eq.GAN_general}):\\
\quad $f(u)=\ln(u)$, $g(u)=\ln(1-u)$,\\
\quad or $f(u)=-\ln(\frac{1}{u})$, $g(u)=-\ln(\frac{1}{1-u})$.
}
\\
\hline
\tabincell{l}{LSGAN\\ \cite{Least-squares-generative-adversarial-networks}\\}
&
\tabincell{l}{
$\max\limits_{G} \min\limits_{D} \big\{ \frac{1}{2}\mathbb{E}_{{\bm x}\sim \mathbb{P}}[(D({\bm x})-1)^2]+ \frac{1}{2}\mathbb{E}_{{\bm z} \sim \mathbb{P}_{z}}[(D(G({\bm z})))^2] \big\}
$\\
}
&
\tabincell{l}{
$\bullet$ $f(\cdot)$ and $g(\cdot)$ in Eq. (\ref{eq.GAN_general}):\\
\quad $f(u)=-\frac{1}{2}(u-1)^2$ and $g(u)=-\frac{1}{2} u^2$.
}
\\
\hline
\tabincell{l}{WGAN \\ \cite{Wasserstein-GAN}\\}
&
\tabincell{l}{
$\min\limits_{G} \max\limits_{D} \big\{ \sup\limits_{\|D\|_{L}\le 1} \{\mathbb{E}_{{\bm x}\sim \mathbb{P}}[D({\bm x})] - \mathbb{E}_{{\bm z}\sim \mathbb{P}_{z}}[D(G({\bm z}))]\} \big\}$\text{,}\\
where $\|\cdot\|_{L}\le 1$ denotes $1$-Lipschitz constraint condition.
}
&
\tabincell{l}{
$\bullet$ $f(\cdot)$ and $g(\cdot)$ in Eq. (\ref{eq.GAN_general}):\\
\quad $f(u)=u$ and $g(u)=-u$.
}
\\
\hline
\tabincell{l}{MIM-based GAN\\ $[\text{Ours}]$\\}
&
\tabincell{l}{
$\max\limits_{G} \min\limits_{D} \big\{ \mathbb{E}_{{\bm x}\sim \mathbb{P}}[\exp(1-D({\bm x}))]
    +\mathbb{E}_{{\bm z}\sim \mathbb{P}_{z}}[\exp(D(G({\bm z})))] \big\}$\\
}
&
\tabincell{l}{
$\bullet$ $f(\cdot)$ and $g(\cdot)$ in Eq. (\ref{eq.GAN_general}):\\
\quad $f(u)=-\exp(1-u)$ and $g(u)=-\exp(u)$.
}
\\
\toprule
\end{tabular}}
\end{table}

\subsection{MIM-based GAN and its characteristics}\label{subsection.MIM-GAN}
Compared with Shannon entropy, the MIM using the exponential function to supersede the logarithmic function \cite{Message-importance-measure-and-its-application}, makes a more positive effect on the small probability event processing.
By virtue of the exponential function, an information distance named Message Identification (M-I) divergence performs better than KLD (depending on the logarithmic function) when detecting outlier sequences \cite{Amplifying-inter-message-distance}.
Consequently, with respect to information characterization, the exponential function has different properties from the logarithmic function.

In view of the amplification for small probability events in the MIM, there may exist a potential superiority for the objective function in which the logarithmic function is replaced by the exponential function.
According to the general optimization form for GANs described in Eq. (\ref{eq.GAN_general_optimizaition}), it is necessary to consider the convexity of objective function, when introducing the exponential function into the optimization problem.
In the light of the convexity of exponential function, the maxmin optimization problem is used as an equivalent form for the minmax optimization problem by adjusting the functions $f(\cdot)$ and $g(\cdot)$ in Eq. (\ref{eq.GAN_general}).

Based on the above discussion, a new kind of objective function with the exponential information metric is presented and its corresponding GAN defined as follows.

\begin{defn}[\textit{MIM-based GAN}]\label{defn.MIM_based_GAN}
By resorting to the idea of the information metric with the exponential function, i.e. MIM, a new adversarial networks named \textit{MIM-based GAN} is designed, whose two-player game with the objective function based on MIM is described as
\begin{equation}\label{eq.opt_LMIM}
\begin{aligned}
     \max_{G} \min_{D} L_{\text{\rm MIM}}(D,G)\text{,}
\end{aligned}
\end{equation}
where $L_{\text{\rm MIM}}(D,G)$ is the objective function given by
\begin{equation}\label{eq.LMIM}
\begin{aligned}
    & L_{\text{\rm MIM}}(D,G)\\
    & = \mathbb{E}_{{\bm x}\sim \mathbb{P}}[\exp(1-D({\bm x}))]
    +\mathbb{E}_{{\bm z}\sim \mathbb{P}_{z}}[\exp(D(G({\bm z})))]\\
    & = \mathbb{E}_{{\bm x}\sim \mathbb{P}}[\exp(1-D({\bm x}))]
    +\mathbb{E}_{{\bm x}\sim \mathbb{P}_{g_{\theta}}}[\exp(D({\bm x}))]\text{,}
\end{aligned}
\end{equation}
whose notations (e.g. $D$, $G$, ${\bm x}$, ${\bm z}$, $\mathbb{P}$, $\mathbb{P}_{g_{\theta}}$ and $\mathbb{P}_{z}$) are the same as those described in Eq. (\ref{eq.GAN_general}).
\end{defn}

Actually, the principle of MIM-based GAN is the same as that in the original GAN, where the objective function leads the neural networks of discriminator and generator to update the weight parameters by using back propagation.

Some fundamental characteristics of MIM-based GAN will be discussed in the following subsections.

\subsubsection{Optimality of $\mathbb{P}=\mathbb{P}_{g_{\theta}}$}\label{section.opt}\ \par

As for a given generator $G$, its corresponding optimal discriminator $D$ is discussed in the following Lemma.
\begin{lem}[\textit{Optimal discriminator}]\label{lem.Optimality}
If the generator $g_{\bm \theta}$ is fixed, the optimal discriminator of MIM-based GAN is given by
\begin{equation}\label{eq.D*}
\begin{aligned}
    D^{*}_{\text{\rm MIM}}({\bm x})
    =\frac{1}{2}+\frac{1}{2}\ln \big\{\frac{P({\bm x})}{P_{g_{\theta}}({\bm x})} \big\}\text{,}
\end{aligned}
\end{equation}
where $P(\cdot)$ and $P_{g_{\theta}}(\cdot)$ respectively denote the real probability density and the generative probability density for the corresponding distributions $\mathbb{P}$ and $\mathbb{P}_{g_{\theta}}$.
\end{lem}
\begin{proof}
Please see the Appendix \ref{app.lem.Optimality}.
\end{proof}

By substituting the optimal discriminator $D^{*}_{\text{MIM}}$ into $L_{\text{MIM}}(D,G)$ under the condition that the generator $G$ is given, it is not difficult to see that
\begin{equation}\label{eq.L_Doptimal}
\begin{aligned}
    & L_{\text{MIM}}(D=D^{*}_{\text{MIM}},G)\\
    & =  \mathbb{E}_{{\bm x}\sim \mathbb{P}}[\exp(1-D^{*}_{\text{MIM}}({\bm x}))]
    +\mathbb{E}_{{\bm x}\sim \mathbb{P}_{g_{\theta}}}[\exp(D^{*}_{\text{MIM}}({\bm x}))]\\
    & = \mathbb{E}_{{\bm x}\sim \mathbb{P}}
    \Big[\exp\Big(\frac{1}{2}+\ln \Big(\frac{P({\bm x})}{P_{g_{\theta}}({\bm x})}\Big)^{-\frac{1}{2}}\Big)\Big]\\
    & \quad + \mathbb{E}_{{\bm x}\sim \mathbb{P}_{g_{\theta}}}
    \Big[\exp\Big(\frac{1}{2}+\ln \Big(\frac{P({\bm x})}{P_{g_{\theta}}({\bm x})}\Big)^{\frac{1}{2}}\Big)\Big]\\
    & = \sqrt{{\rm e}} \bigg\{
     \mathbb{E}_{{\bm x}\sim \mathbb{P}}\bigg[\bigg(\frac{P({\bm x})}{P_{g_{\theta}}({\bm x})}\bigg)^{-\frac{1}{2}}\bigg]
    +\mathbb{E}_{{\bm x}\sim \mathbb{P}_{g_{\theta}}}\bigg[\bigg(\frac{P_{g_{\theta}}({\bm x})}{P({\bm x})}\bigg)^{-\frac{1}{2}}\bigg]
    \bigg\}\text{.}
\end{aligned}
\end{equation}

\begin{prop}[\textit{Optimal solution of objective function}]\label{prop.maximum}
Assume that the optimal discriminator in MIM-based GAN is obtained, which is given by Eq. (\ref{eq.D*}).
In this case, the equivalent objective function i.e. $L_{\text{\rm MIM}}(D=D^{*}_{\text{\rm MIM}},G)$ (mentioned in Eq. (\ref{eq.L_Doptimal})) achieves its optimal solution if and only if $\mathbb{P}=\mathbb{P}_{g_{\theta}}$,
where the maximum value of $L_{\text{\rm MIM}}(D=D^{*}_{\text{\rm MIM}},G)$ is reached, that is $2\sqrt{{\rm e}}$.
\end{prop}
\begin{proof}
Please see the Appendix \ref{app.prop.maximum}.
\end{proof}

\subsubsection{Resistance to mode collapse}\label{section.mode_collapse}\ \par

It is worth analyzing the mode collapse which likely appears in the original GAN when optimizing the corresponding objective function.
To simplify the theoretical analysis, we focus on the generator which has an immediate impact on data generation by supposing that the optimal discriminator is given.

\begin{prop}[\textit{Mode collapse resistance}]\label{prop.mode_collapse}
Assume that with respect to an arbitrary generator, the corresponding optimal discriminator is achieved in the two-player game of MIM-based GAN.
In this case, the optimization problem for the generator which dominates the efficiency of generating data, prevents it from $P_{g_{\theta}}({\bm x}) \to 0$, where $P_{g_{\theta}}({\bm x})$ is the generative probability density.
This implies more efficient data categories involved in the real probability density $P({\bm x})$ are probably generated rather than ignored, which means a kind of resistance to mode collapse.
Similar to MIM-based GAN, this mode collapse resisting capacity is also available in LSGAN and WGAN.
However, in the original GAN, the optimization for generator leads to $P_{g_{\theta}}({\bm x}) \to 0$ regardless of any $P({\bm x})$, which causes the mode collapse.

\end{prop}
\begin{proof}
Please see the Appendix \ref{app.prop.mode_collapse}.
\end{proof}

\subsubsection{Interference on the gradient of generator}\ \par

In the process of adversarial training for GANs, it is apparent that the discriminator makes an effect on the generator.
When unexpected disturbances appear in the discriminator, the generator might be pulled into the unstable training state.
As a result, it is worth investigating the anti-interference capability of generator with respect to the disturbances in the discriminator.
In particular, we discuss this characteristic under the condition that the generative data does not appear in the real data.
This implies that the generative data is considered as fake data entirely.

\begin{prop}[\textit{Generator gradient interfered by the discriminator}]\label{prop.stability}
Let $g_{\bm \theta} : \mathcal{Z} \to \mathcal{X}$ be a differentiable
function to generate data similar to real data.
The real distribution and generative distribution are denoted by $\mathbb{P}$ and $\mathbb{P}_{g_{\theta}}$, respectively.
Let $\mathbb{P}_{z}$ be the distribution for the random input ${\bm z}$ and $D$ be a discriminator ($D\in [0,1]$).
As for MIM-based GAN, the stability for the generator gradient is discussed based on the condition that
the generative data is treated as fake data entirely, where the corresponding perfect discriminator is denoted by $\tilde D^*$ and it is seen that $\tilde D^*(g_{\bm\theta}({\bm z}))=0$.
Assuming $D-\tilde D^*=\delta$ and
$\nabla_{{\bm x}} D({\bm x})-\nabla_{{\bm x}} \tilde D^*({\bm x})=\upsilon$ where
$\nabla_{{\bm x}} \tilde D^*({\bm x}) =0 $, we have
\begin{equation}\label{eq.MIM_perfect_stability}
\begin{aligned}
    & \nabla_{\bm\theta}\mathbb{E}_{{\bm z}\sim \mathbb{P}_z} [\exp(D(g_{\bm\theta}({\bm z}))) ]
    - \nabla_{\bm\theta}\mathbb{E}_{{\bm z}\sim \mathbb{P}_z} [\exp(\tilde D^*(g_{\bm\theta}({\bm z}))) ] \\
    & = \exp(\delta) \upsilon \mathbb{E}_{{\bm z}\sim \mathbb{P}_z}
    [\nabla_{\bm\theta}g_{\bm\theta}({\bm z}) ]\text{,}
\end{aligned}
\end{equation}
where $\delta$ and $\upsilon$ are very small disturbances, 
as well as $\nabla_{\bm\theta}\mathbb{E}_{{\bm z}\sim \mathbb{P}_z} [\exp(D(g_{\bm\theta}({\bm z}))) ]$ represents the gradient of generator.

\end{prop}
\begin{proof}
Please see the Appendix \ref{app.prop.stability}.
\end{proof}

Additionally, the perfect discriminator $\tilde D^*$ satisfying $\tilde D^*(g_{\bm\theta}({\bm z}))=0$ in Proposition \ref{prop.stability}, is consistent with the optimal discriminator which is given by
$D^{*}_{\text{\rm MIM}}({\bm x})=\frac{1}{2}+\frac{1}{2}\ln \big\{\frac{P({\bm x})}{P_{g_{\theta}}({\bm x})} \big\}$ in MIM-based GAN.
Specifically, if we have the real data probability satisfying $P({\bm x})\to 0$ and generative data probability satisfying $P_{g_{\theta}}({\bm x}) \neq 0$, it is readily seen that
$D^{*}_{\text{\rm MIM}}\to 0$ where $D^{*}_{\text{\rm MIM}}\in [0,1]$.

\begin{cor}\label{cor.stability}
Let $g_{\bm\theta} : \mathcal{Z} \to \mathcal{X}$ be a differentiable function that is adopted to produce the data following the generative distribution $\mathbb{P}_{g_{\theta}}$.
Let $\mathbb{P}_z$ be the distribution for the random input ${\bm z}$ (${\bm z}\in \mathcal{Z}$), $\mathbb{P}$ be the real distribution, as well as $D$ be the discriminator (where $D\in [0,1]$).
If it is satisfied that $D-\tilde D^*=\delta$ and $\nabla_{{\bm x}} D({\bm x})-\nabla_{{\bm x}} \tilde D^*({\bm x})=\upsilon$
where $\tilde D^*(g_{\bm\theta}({\bm z}))=0$ and $\nabla_{{\bm x}} \tilde D^*({\bm x}) =0 $, as well as
$\tilde D^*$ is regarded as the perfect discriminator, we have
\begin{equation}
\begin{aligned}
 & \frac{| \nabla_{\bm\theta} L_{\text {\rm KL}}(D,g_{\bm\theta}) - \nabla_{\bm\theta} L_{\text {\rm KL}}(\tilde D^*,g_{\bm\theta}) |}
 {\big| \nabla_{\bm\theta} L_{\text {\rm MIM}}(D,g_{\bm\theta}) - \nabla_{\bm\theta} L_{\text {\rm MIM}}(\tilde D^*,g_{\bm\theta}) \big|} 
 > 1,
\end{aligned}
\end{equation}
where $\nabla_{\bm\theta} L_{\text {\rm KL}}(D,g_{\bm\theta}) - \nabla_{\bm\theta} L_{\text {\rm KL}}(\tilde D^*,g_{\bm\theta})$ and
$\nabla_{\bm\theta} L_{\text {\rm MIM}}(D,g_{\bm\theta}) - \nabla_{\bm\theta} L_{\text {\rm MIM}}(\tilde D^*,g_{\bm\theta})$
are the gradient differences 
in the original GAN and MIM-based GAN, respectively.
This implies that MIM-based GAN has stronger anti-interference capability than the original GAN in terms of the gradient of generator.
\end{cor}
\begin{proof}
Please see the Appendix \ref{app.cor.stability}.
\end{proof}

In addition, as for the LSGAN and WGAN, it is not difficult to see that $| \nabla_{\bm\theta} L_{\text {\rm LS}}(D,g_{\bm\theta}) - \nabla_{\bm\theta} L_{\text {\rm LS}}(\tilde D^*,g_{\bm\theta}) |
= | \nabla_{\bm\theta} L_{\text {\rm W}}(D,g_{\bm\theta}) - \nabla_{\bm\theta} L_{\text {\rm W}}(\tilde D^*,g_{\bm\theta}) |
= \upsilon \mathbb{E}_{{\bm z}\sim \mathbb{P}_z}[\nabla_{\bm\theta}g_{\bm\theta}({\bm z}) ] $, where the notations are similar to those in Corollary \ref{cor.stability}.
This indicates that in the LSGAN and WGAN, the interference on the gradient of generator is only related to $\upsilon$ and not to $\delta$, which is different from that in the original GAN and MIM-based GAN.

\subsection{Small probability event analysis}\label{section.rare_events_generator}

In this subsection, we shall investigate how small probability event importance is amplified in MIM-based GAN.

In the training process of GANs, a near good discriminator is trained, which leads to achieve the optimal generator under this given discriminator.
As a feedback, the conditional optimal generator also helps to train a better discriminator.
This process runs iteratively until the two adversarial networks reach the equilibrium state \cite{How-GANs,GANs-trained}.

According to the two-player game of GANs, in the case that there is an ideal discriminator, our goal is to select an appropriate generator as the optimal solution for the objective function.
In the original GAN, low occurrence events play minor parts in the value of objective function.
This indicates that smaller probability events are easily ignored when maximizing the major part of the objective function.
Consequently, it is necessary to investigate how much small probability events affect the value of objective function.

In fact, the neural networks of GANs are guided by the objective function to generate fraudulent data.
The events reflected to larger proportion in the objective function, will have a larger contribution on the value of objective function, so that these events make more effects to guide the data generation.
This implies that the generation capability for small probability events is related to the corresponding proportion of these events in the objective function.

Furthermore, it is also necessary to describe the relationship between the real distribution and generative distribution.
Similar to the way in \cite{An-information-theoretic-approach}, the corresponding relationship is given by
\begin{equation}\label{eq.Pg_Pr}
\left\{
\begin{aligned}
    & P_{g_{\theta}}({\bm x})=P({\bm x}) + \varepsilon P^{\gamma}({\bm x}) \rho({\bm x}),\\
    & \text{s.t.} \int_{\mathcal{X}} \varepsilon P^{\gamma}({\bm x}) \rho({\bm x}) {\rm d}{\bm x} = 0,
\end{aligned}
\right.
\end{equation}
in which $\varepsilon$ and $\gamma$ denote a small disturbance parameter and an adjustable parameter (regarded as a constant), respectively, $\rho({\bm x})$ is a perturbation function, $\mathcal{X}$ is the domain of the variable ${\bm x}$, as well as $P_{g_{\theta}}({\bm x})$ and $P({\bm x})$ denote the probability densities for the real distribution and generative distribution, respectively.

\begin{prop}[\textit{The influence of small probability events in the objective function}]\label{prop.proportion_rare_events_MIM}
Assume there is a real data distribution $\mathbb{P}$ whose probability density is given by $P({\bm x})$, where the probability events are casted into two class sets, namely small probability event set $\Omega_{\text{\rm small}}$ and large probability event set $\Omega_{\text{\rm large}}$.
Let $g_{\bm\theta} : \mathcal{Z} \to \mathcal{X}$ be a differentiable
function that is used to generate data whose probability distribution and density are denoted by $\mathbb{P}_{g_{\theta}}$ and $P_{g_{\theta}}({\bm x})$, respectively.
Besides, the relationship between $\mathbb{P}$ and $\mathbb{P}_{g_{\theta}}$ is consistent with
that described in Eq. (\ref{eq.Pg_Pr}).
This implies a generative event and its corresponding real event belong to the same probability event set (i.e. $\Omega_{\text{\rm small}}$ or $\Omega_{\text{\rm large}}$).
As well, assume it is satisfied that $\int_{\Omega_{\text{\rm small}}} P({\bm x}) {\rm d}{\bm x} \ll \frac{1}{2}$ and $\int_{\Omega_{\text{\rm small}}} P_{g_{\theta}}({\bm x}) {\rm d}{\bm x} \ll \frac{1}{2}$.
When the optimal discriminator of MIM-based GAN is achieved, namely $D^{*}_{\text{\rm MIM}}({\bm x})=\frac{1}{2}+\frac{1}{2}\ln\frac{P({\bm x})}{P_{g_{\theta}}({\bm x})}$, the proportion of small probability events in the value of objective function is given by
\begin{equation}\label{eq.R_small_MIM}
\begin{aligned}
    & \Upsilon_{\Omega_{\text{\rm small}}}^{\text {\rm MIM}} \\
    & \approx \frac{\frac{ \Phi\{{\bm x}\in \Omega_{\text{\rm small}}\}
    + \Phi_{g_{\theta}}\{{\bm x}\in \Omega_{\text{\rm small}}\}}{2}
    -\frac{1}{8}\varepsilon^2 \int_{\Omega_{\text{\rm small}}} P^{2\gamma-1}({\bm x})\rho^2({\bm x})
    {\rm d}{\bm x} }
    { 1-\frac{1}{8}\varepsilon^2 \int_{\mathcal{X}} P^{2\gamma-1}({\bm x})\rho^2({\bm x})
    {\rm d}{\bm x} }\text{,}
\end{aligned}
\end{equation}
where $\Phi\{{\bm x}\in \Omega_{\text{\rm small}}\}
= \int_{\Omega_{\text{\rm small}}} P({\bm x}) {\rm d}{\bm x}$,
$\Phi_{g_{\theta}}\{{\bm x}\in \Omega_{\text{\rm small}}\}
= \int_{\Omega_{\text{\rm small}}} P_{g_{\theta}}({\bm x}) {\rm d}{\bm x}
= \int_{\Omega_{\text{\rm small}}} \{ P({\bm x}) + \varepsilon p^{\gamma}({\bm x}) \rho({\bm x}) \}
{\rm d}{\bm x}$,
$\mathcal{X}$ is the domain of ${\bm x}$, i.e. $\mathcal{X}= \Omega_{\text{\rm small}} \cup \Omega_{\text{\rm small}}$,
$\gamma$ denotes an adjustable parameter (treated as a constant) and $\varepsilon$ denotes a small enough disturbance parameter.

In particular, if the real distribution is a Bernoulli distribution described as $\mathbb{P}= \{p,1-p\}$ ($0<p\ll \frac{1}{2}$), as well as the generative distribution is given by $\mathbb{P}_{g_{\theta}}= \{ p+\varepsilon p^{\gamma}\rho_p, 1-p-\varepsilon p^{\gamma}\rho_p \} = \{q,1-q\}$ ($q < \frac{1}{2}$), then we have
\begin{equation}
\begin{aligned}
    \Upsilon_{\Omega_{\text{\rm small}}}^{\text {\rm MIM}}
    & \approx \frac{ \frac{p+q}{2}-\frac{1}{8}\varepsilon^2p^{2\gamma-1}{\rho_p^2} }
    { 1-\frac{1}{8} \varepsilon^2 \frac{ p^{2\gamma-1}{\rho_p^2}}{1-p}}\text{,}
\end{aligned}
\end{equation}
in which the parameters $\varepsilon$ and $\gamma$ are the same as those in Eq. (\ref{eq.R_small_MIM}).
\end{prop}
\begin{proof}
Please see the Appendix \ref{app.prop.proportion_rare_events_MIM}.
\end{proof}

\begin{cor}\label{cor.proportion_rare_events}
Let $\mathbb{P}$ and  $\mathbb{P}_{g_{\theta}}$ be the real distribution and the generative one,
whose relationship is described as that in Eq. (\ref{eq.Pg_Pr}).
Assume the optimal discriminator is achieved, which implies $D^{*}_{\text{\rm MIM}}({\bm x})=\frac{1}{2}+\frac{1}{2}\ln\frac{P({\bm x})}{P_{g_{\theta}}({\bm x})}$ for MIM-based GAN as well as $D^{*}_{\text{\rm KL}}({\bm x})=D^{*}_{\text{\rm LS}}({\bm x})=\frac{P({\bm x})}{P({\bm x})+P_{g_{\theta}}({\bm x})}$ for the original GAN and LSGAN.
Consider the case that $\int_{\Omega_{\text{\rm small}}} P({\bm x}) {\rm d}{\bm x} \ll \frac{1}{2}$, $\int_{\Omega_{\text{\rm small}}} P_{g_{\theta}}({\bm x}) {\rm d}{\bm x} \ll \frac{1}{2}$ and $\int_{\Omega_{\text{\rm small}}} P({\bm x}) {\rm d}{\bm x} = O(\varepsilon)$ ($\varepsilon$ is a small disturbance parameter, the same as that mentioned in Eq. (\ref{eq.Pg_Pr})).
Then, compared with the original GAN and LSGAN, small probability events have greater effects on the value of the objective function in MIM-based GAN, namely
\begin{equation}\label{eq.Lambda_MIM_KL}
    \Upsilon_{\Omega_{\text{\rm small}}}^{\text {\rm MIM}}
    \ge \Upsilon_{\Omega_{\text{\rm small}}}^{\text {\rm KL}},
\end{equation}
as well as
\begin{equation}\label{eq.Lambda_MIM_LS}
        \Upsilon_{\Omega_{\text{\rm small}}}^{\text {\rm MIM}}
    \ge \Upsilon_{\Omega_{\text{\rm small}}}^{\text {\rm LS}},
\end{equation}
where the equality is held if the condition $p=q$ is satisfied.
Besides, $\Upsilon_{\Omega_{\text{\rm small}}}^{\text {\rm KL}}$, $\Upsilon_{\Omega_{\text{\rm small}}}^{\text {\rm LS}}$ and $\Upsilon_{\Omega_{\text{\rm small}}}^{\text {\rm MIM}}$ denote the proportion of small probability events in the value of objective function with respect to the original GAN, LSGAN and MIM-based GAN, respectively.
\end{cor}

\begin{proof}
Please see the Appendix \ref{app.cor.proportion_rare_events}.
\end{proof}

Additionally, as for WGAN, its optimal discriminator
depends on the positive or negative sign of $[P({\bm x})-P_{g_{\theta}}({\bm x})]$ rather than the values of $P({\bm x})$ and $P_{g_{\theta}}({\bm x})$.
In this regard, the corresponding objective function holds uncertain proportion of small probability events.
Therefore, when comparing MIM-based GAN with WGAN, it is not sure which one has larger proportion of small probability events in the objective function.

\section{Unsupervised anomaly detection with MIM-based GAN}\label{section.anomaly_detection_MIM}
As a technique potentially used in IoT, GAN-based anomaly detection is an intersection between the probability event generation and multi-sensor data processing.

In the conventional GAN-based anomaly detection methods, such as AnoGAN \cite{Unsupervised-anomaly}, ALAD \cite{Adversarially-learned}, GANomaly \cite{GANomaly} and f-AnoGAN \cite{f-AnoGAN}, normal data is used to train the GAN-based models so that the outputs of generator are viewed as the normal.
By means of identification tools such as Euclidean distance, the anomalous events in the testing data set are dug out due to their striking differences from the generative events.
However, it is almost impossible for the generator to output all the normal events, especially those with small probability, since the small probability events with low proportion in the objective function are easily ignored.

Here, we extend the training data of GAN-based anomaly detection from normal data to the whole unsupervised data (including both normal and anomalous data), whose principle is discussed in the following section.

\subsection{Principle analysis}\label{section.Analysis and discussion}\ \par

On one hand, the data generation preference is discussed with respect to MIM-based GAN.
As for the trained generator $G$, it is viewed as a mapping from a latent data space to the real data space, i.e. $G : \mathcal{Z} \to \mathcal{X}$.
Considering that the normal data occurring frequently is more likely learned,
the outputs of generator prefer to reflect the principal components for the distribution of real data, which is described in details as follows.

\begin{prop}[\textit{Preference to generate principal components of distributions}]\label{prop.R_large_probability}
Assume that the optimal discriminator described as Eq. (\ref{eq.D*}) is obtained for a given generator in MIM-based GAN.
Consider the case that the generative distribution $\mathbb{P}_{g_{\theta}}$ is close to the real distribution $\mathbb{P}$. That is, $\varepsilon$ is small enough or even $\varepsilon \to 0$ in the relationship described as Eq (\ref{eq.Pg_Pr}).
As well, the real data is classified into two classes, namely the large probability event set $\Omega_{\text{\rm large}}$ and the small probability event set $\Omega_{\text{\rm small}}$.
Furthermore, if the small probability events satisfy $\int_{\Omega_{\text{\rm small}}}{P({\bm x})} {\rm d}{\bm x} \le \zeta$ (where it is satisfied that $\zeta \ll 1$), we have the proportion of large probability events in the objective function given by
\begin{equation}
\begin{aligned}
\Upsilon_{\Omega_{\text{\rm large}}}^{\text {\rm MIM}}
& \ge
\frac{ 1-\zeta + O(\varepsilon) }
{1+ O(\varepsilon)}
\overset{\varepsilon \to 0}{\longrightarrow}
1-\zeta,
\end{aligned}
\end{equation}
where $O(\cdot)$ denotes the equivalent infinitesimal.
This implies that the principle components of the distribution lying in the large probability event set dominate the value of the objective function, which are more likely generated.
\end{prop}
\begin{proof}
Please see the Appendix \ref{app.prop.R_large_probability}.
\end{proof}

Actually, since the optimal solution of objective function lies in $\mathbb{P}=\mathbb{P}_{g_{\theta}}$ mentioned in Proposition \ref{prop.maximum},
it is reasonably satisfied that $\varepsilon$ is small enough or even $\varepsilon \to 0$ in the relationship described as Eq (\ref{eq.Pg_Pr}), when the generator is drawn to be close to the optimal state.
That is, the generative probability $P_{g_{\theta}}({\bm x})$ is similar to the real $P({\bm x})$.
This implies $P_{g_{\theta}}({\bm x}) \approx P({\bm x}) \gg 0 $ in the region $\{{\bm x} \in \Omega_{\text{large}}\}$,
while in the region $\{{\bm x} \in \Omega_{\text{small}}\}$, $P_{g_{\theta}}({\bm x})$ is quite small or even approximate to zero.

In the above case, we have the proportion of large probability events in the objective function approximating to the probability of large probability events, which is close to one if the proportion of small probability events is small enough.
This also corresponds to that $\Upsilon_{\Omega_{\text{\rm large}}}^{\text {\rm MIM}} \to 1$ if $\zeta \to 0$ in Proposition \ref{prop.R_large_probability}.

Therefore, it is readily seen that large probability events belonging to the principle components of distribution make vital roles in the value of objective function.
This leads the generator to tend to generate large probability events more efficiently.

On the other hand, unsupervised anomaly detection with MIM-based GAN is discussed.
At first, we train MIM-based GAN without supervised labels. In this case, the generative samples output from the generator can be viewed as normal events, which is analyzed as follows.

According to Proposition \ref{prop.R_large_probability}, the generative data $G({\bm z})$ is more likely cast into the large probability event set.
In fact, it is usual that the whole large probability event set holds the main part of normal event set.
In this case, the large probability event set (including most of the normal events) has a greater impact on the generator than the small probability event set (that usually is composed of a few normal events with smaller probability and the anomalous events).
Thus, the generative samples tend to lie in the dominant part of normal event set, which can be treated as a representation for the whole normal event set in the anomaly detection.

Then, the trained model for generating samples is used to help dig out anomalies.
Specifically, when a measurement such as $\mathcal{L}_2$ distance is chosen, some testing samples which are not close to the generative samples, are regarded as anomalies. The details of this method are given in Section \ref{section.procedure_detection_GAN}.

\begin{figure}[!htb]
\centering
\includegraphics[width=6.5 in]{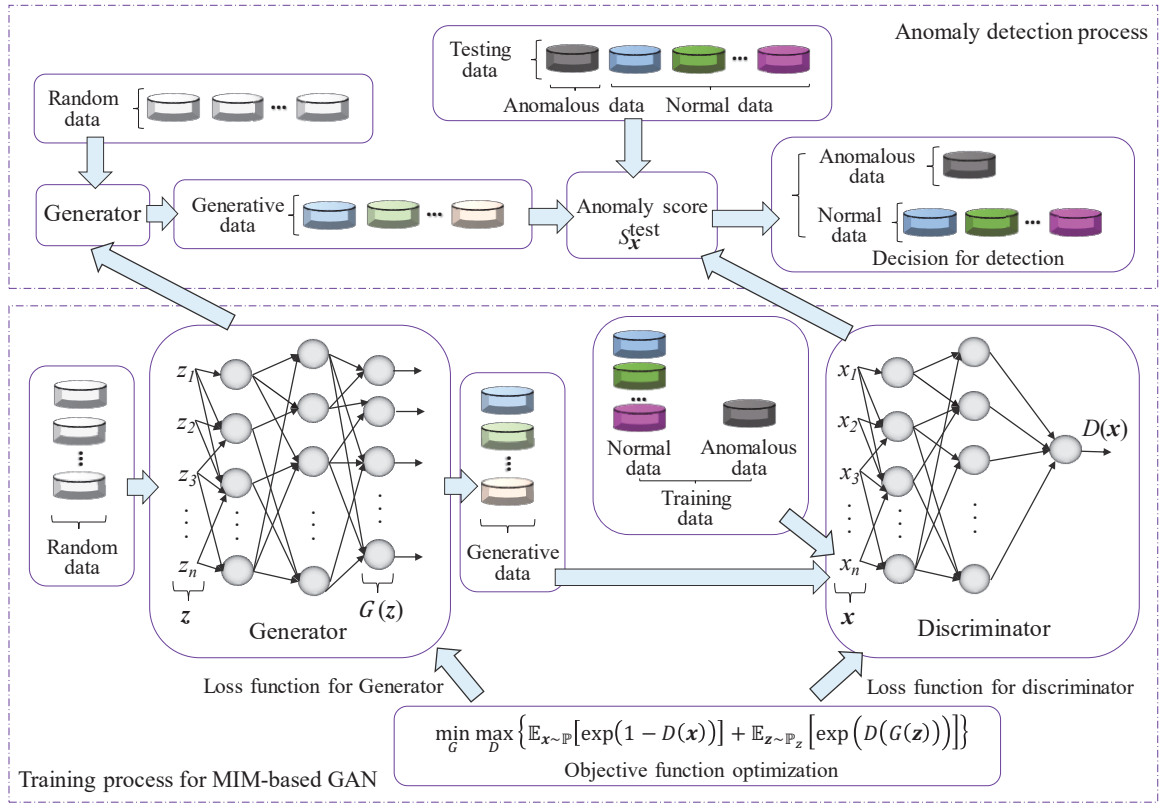}
\vspace{-0.8cm}
\caption{The framework for the unsupervised anomaly detection with MIM-based GAN.}
\label{fig_detect_model}
\end{figure}
\subsection{MIM-based GAN method for unsupervised anomaly detection}\label{section.procedure_detection_GAN}

Here, we shall introduce a method for the unsupervised anomaly detection with MIM-based GAN.
The primary aim is to allocate labels $\{$``$0$'',``$1$''$\}$ to testing samples, where ``$0$'' and ``$1$'' denote the normal event label and the anomalous event label, respectively.
The framework is shown as Fig. \ref{fig_detect_model},
whose key point lies in the part of the training process for MIM-based GAN.
The details are given as follows. 

\begin{itemize}
\item{\textbf{Data preparation and model training}}

First of all, we train MIM-based GAN to generate the data approximating to real data.
The real data without labeling is divided into two parts, namely training data and testing data, corresponding to the two data sets $\Omega_{\text{train}}$ and $\Omega_{\text{test}}$, respectively.
The inputs for MIM-based GAN consist of real data ${\bm x}$ (containing normal data and anomalous data) for the discriminator and random data ${\bm z}$ in the latent space for the generator.
Then, the maxmin game mentioned in Eq. (\ref{eq.opt_LMIM}) is used to train the generator and discriminator (i.e. two neural networks) alternately. 
After enough training iterations, a well trained discriminator $D$ and its corresponding generator $G$ are obtained.

\item{\textbf{Anomaly score computing}}

In this step, the anomaly score is designed as a detection measurement tool.
Generally speaking, it mainly makes use of the well trained discriminator $D$ and generator $G$ to compute the $p$-norm distance and its regularization, where $D$ is used to deal with testing data and generative data, as well as $G$ with random data as its input is used to produce the generative data.
Specifically, the anomaly score is given by
\begin{equation}\label{eq.anomaly_score}
    S^{\text{test}}_{\bm x} = ( 1- \eta )J_{\text{error}}({\bm x}, {\bm z}_{\text{opt}}) + \eta H_{\text{ce}}( D({\bm x}), \beta)\text{,}
\end{equation}
where ${\bm z}_{\text{opt}} = \arg \min\limits_{{\bm z}} {J_{\text{error}}( {\bm x} , {\bm z} )}$, ${\bm x}\in \Omega_{\text{test}}$, $\eta$ is a small parameter ($0<\eta<1$) and $\beta=1$.
As well, the two main components of $S^{\text{test}}_{\bm x}$, i.e. $H_{\text{ce}}(\cdot,\cdot)$ (namely, sigmoid cross entropy) and $J_{\text{error}}(\cdot,\cdot)$, are respectively given by
\begin{equation}
\begin{aligned}
 & H_{\text{ce}}( D({\bm x}), \beta) \\
 & = -\beta \ln \Big[\frac{1}{1+\exp(-D({\bm x}))}\Big]
 - (1-\beta) \ln \Big[1- \frac{1}{1+\exp(-D({\bm x}))}\Big]\text{,}
\end{aligned}
\end{equation}
as well as
\begin{equation}\label{eq.Lerror}
\begin{aligned}
& {J_{\text{error}}({\bm x} , {\bm z} )} \\
& = (1-\lambda)|| {\bm x} - G({\bm z})||_{p} + \lambda H_{\text{ce}}( D(G({\bm z})), \beta)\text{,}
\end{aligned}
\end{equation}
in which $|| \cdot||_{p}$ denotes the $p-$norm $(\text{usually } p=2)$, $\lambda$ is a weight factor $(0<\lambda<1)$ and $H_{\text{ce}}( D(G({\bm z})), \beta)$ is a regularization to ensure $G({\bm z})$ in the real data space.
Then, we normalize these anomaly scores and go to the next step.

\item{\textbf{Decision for detection}}

By resorting to the anomaly score given by Eq. (\ref{eq.anomaly_score}), we make a decision and provide the label for each sample in the testing dataset by using
\begin{equation}\label{eq.decision}
A^{\text{test}}_{\bm x}=\left\{
\begin{aligned}
1, \quad \text{for} \quad S^{\text{test }}_{\bm x}> \Gamma\text{,}\\
0, \quad \text{otherwise}\text{,} \qquad
\end{aligned}
\right.
\end{equation}
where $A^{\text{test}}_{\bm x}$ is the label for testing samples and its non-zero value indicates a detected anomalous event.
In other words, when the scores are higher than a predefined threshold $\Gamma$, the corresponding samples will be considered as anomalies.
\end{itemize}

Moreover, a summary of the above pipeline is provided in Algorithm \ref{Algorithm.anomaly_detection}.

\begin{algorithm}
\caption{Unsupervised anomaly detection with MIM-based GAN}\label{Algorithm.anomaly_detection}
\begin{algorithmic}[1]
\REQUIRE Real data ${\bm x}$ without labeling (containing normal data and anomalous data) for discriminator as well as random data ${\bm z}$ for generator.
\ENSURE The predicted label $A^{\text{test}}_{\bm x}$ (namely, normal class label or anomalous class label) for each testing sample.
\STATE Real data set $\rightarrow$ training data set $\Omega_{\text{train}}$ and testing data set $\Omega_{\text{test}}$.
\STATE Train the discriminator $D$ and generator $G$ of MIM-based GAN jointly.
\FORALL {${\bm x}\in \Omega_{\text{test}}\ $}
    \STATE ${\bm z}_{\text{opt}} \leftarrow  \arg \min\limits_{{\bm z}} {J_{\text{error}}( {\bm x} , {\bm z} )}$, where $J_{\text{error}}(\cdot, \cdot)$ is given by Eq. (\ref{eq.Lerror}).
    \STATE  $S^{\text{test}}_{\bm x} \leftarrow ( 1- \eta )J_{\text{error}}({\bm x}, {\bm z}_{\text{opt}}) + \eta H_{\text{ce}}( D({\bm x}), \beta)$.
\ENDFOR
\STATE $S^{\text{test}}_{\bm x} \leftarrow
\frac{ S^{\text{test}}_{\bm x}-\min\{S^{\text{test}}_{\bm x}\} }
{ \max\{S^{\text{test}}_{\bm x}\}-\min\{S^{\text{test}}_{\bm x}\} } $.
\STATE $A^{\text{test}}_{\bm x} \leftarrow$ Eq. (\ref{eq.decision}).
\end{algorithmic}
\end{algorithm}

Essentially, as is mentioned above, the anomaly score is the dominant measurement to detect anomalies.
With respect to its first term, $J_{\text{error}}({\bm x}, {\bm z}_{\text{opt}})$ is used to reflect whether a sample ${\bm x}$ belongs to the anomaly set or not.
As well, it mainly depends on $|| {\bm x} - G({\bm z})||_{p}$ which is minimized to reveal if the sample ${\bm x}$ is close enough to a generative one (viewed as a normal event).
Meanwhile, as for the second term, $H_{\text{ce}}(D({\bm x}), \beta)$ is used as a regularization which measures how far a testing data can be regarded as a training data (i.e. the real data).
Generally speaking, $J_{\text{error}}({\bm x}, {\bm z}_{\text{opt}})$ plays the most important role in the anomaly score.

Furthermore, the core of $J_{\text{error}}({\bm x}, {\bm z}_{\text{opt}})$ lies on the generator $G$ whose generative data has a large influence on the value of the anomaly score.
According to Section \ref{section.Analysis and discussion}, it is explained how the generative data of MIM-based GAN makes sense in the unsupervised anomaly detection, from the viewpoint of event probability.

\subsection{Discussion for IoT applications}

From the perspective of practical applications, anomaly detection has a great impact on the stability and safety of IoT networks, in which there exist increasing number of distributed devices producing a larger amount of heterogeneous data \cite{6G_Enabled,Wireless_Commun}.
Outliers probably appear in IoT systems, which result from system errors or sensor misreadings.
As well, IoT networks are vulnerable to network-based attacks and security threats \cite{Big_Data_IoT}.
Hence, it is significant to identify any anomalous or unusual events in IoT networks, to avoid erroneous information transmission, as well as to improve the system reliability and efficiency.

Moreover, IoT-based monitoring applications also need anomaly detection to distinguish the faults or anomalies from the normal data of the monitored devices or environments \cite{Make-the-rocket,CHRIST}.
Due to diverse varieties of data generated from IoT systems, conventional anomaly detection methods are  not efficient enough to process these multidimensional data.
Therefore, it is a potential application for our MIM-based GAN method to detect anomalies in IoT systems. The details are discussed as follows.

With respect to a general architecture for IoT systems, there exist three key layers, including a perceptual layer, a network layer and an application layer \cite{Big_Data_IoT,6G_Enabled,Make-the-rocket}.
In particular, the perceptual layer with sensor nodes serves for collecting real-world information.
The network layer is responsible for end-to-end device connection, where there also exist IoT edge devices for edge computing and data processing to release resource centralization.
The application layer with cloud servers is the central controller that has high computing power and adequate storage to deal with received data for specific orders and applications.
In this architecture, our MIM-based GAN method for anomaly detection may be carried out in the network layer and the application layer, including the following key points.
\begin{itemize}
\item \textbf{Cloud-based model training}: In the application layer, the IoT cloud servers with enough computing and storage resources works for MIM-based GAN training.
    In this regard, a large amount of historical or updated data collected from the real world is used as input for the training process.
    Then, the trained model parameters are transmitted to the IoT edge computing devices.
\item \textbf{Edge-based anomaly detection}: In the network layer, anomaly detection process is taken at the IoT edge computing devices, whose input is the real-time collected data.
    The detected results are used for the specific operations depending on the corresponding  applications.
\end{itemize}
The above discussion is also shown in Fig. \ref{fig_IoT_model}.

In practice, it is reasonable to deploy our trained model for detection process at IoT edge devices with constrained computing resources, while using IoT cloud to train the model.
Additionally, it is not difficult to transfer our MIM-based GAN model to achieve anomaly detection for IoT-based applications such as distributed intrusion detection \cite{Distributed_Intrusion}, liquid rocket engine monitoring \cite{Make-the-rocket} and water purification system \cite{CHRIST}.
\begin{figure}[!t]
\centering
\includegraphics[width= 0.7\linewidth ]{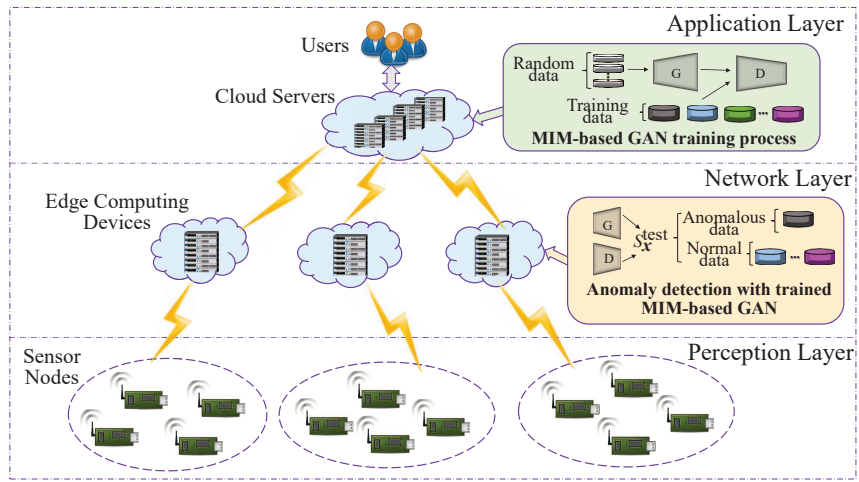}
\caption{The diagram for MIM-based GAN anomaly detection in IoT applications.}
\label{fig_IoT_model}
\end{figure}

\section{Experiments}

Here, we shall take experiments to evaluate the performance of the proposed MIM-based GAN and its corresponding anomaly detection method.

\subsection{Datasets}\label{section.datasets}
Several real datasets from the online repository, i.e. Outlier Detection DataSet (ODDS)
\footnote{http://odds.cs.stonybrook.edu}, are adopted in the experiments, which are introduced as follows.

\begin{itemize}
\item \textbf{Cardiotocography} (or Cardio for short):
As for the Cardiotocography dataset, it has $21$ features in each sample, including Uterine Contraction (UC), Fetal Heart Rate (FHR) and so on.
There exist $1,831$ samples in this classification dataset, in which the pathologic class contains $176$ samples ($9.6\%$ anomalous samples).

\item \textbf{Thyroid}:
The Thyroid dataset includes $6$ real attributes, namely features.
In this dataset, the goal is to dig out the hypothyroid patients regarded as anomalies.
There exist $3,772$ samples in this dataset, including $93$ hyperfunction samples ($2.5\%$ anomalous samples).

\item \textbf{Musk}:
In the Musk dataset, there exist $166$ features in each sample.
Each sample in this dataset can be classified into the several-musks class (regarded as the normal event set) or the non-musk class (regarded as the anomalous event set).
By the way, the number of samples is $3,062$ in this dataset which includes $3.2\%$ anomalies.
\end{itemize}

\subsection{Experiment details}
\subsubsection{Data generation performance}\label{subsection.data_gen_ODDS} \ \par
To reveal the data generation efficiency of MIM-based GAN, we carry out the experiments on the three datasets from ODDS mentioned in Section \ref{section.datasets}.
As for the training data, there are $1,360$ samples (about $74.28 \%$ data), $2,800$ samples (about $74.23\%$ data) and $2,200$ samples (about $71.85\%$ data) drawn randomly from these datasets, respectively.
As well, the rest samples are used as the corresponding testing data.

In terms of the architecture of neural networks, Deep Neural Networks (DNNs) with four layers are chosen, in which sigmoid function and tanh function are chosen as the activation functions in the output layers for the discriminator and generator respectively, while leaky ReLU is chosen for the other layers.
Besides, the Adam algorithm with $0.0001$ learning rate is used to optimize the weights for the networks.

Furthermore, the generator and discriminator in MIM-based GAN are jointly trained to generate data for the above three datasets.
To evaluate the data generation performance, the Reconstruction Error (RE) is selected as a measurement \cite{Pros-cons} which is defined as
\begin{equation}
\begin{aligned}
\mathcal{L}_{\text{RE}} = \frac{1}{N} \sum_{i=1}^{N} \min\limits_{\bm z} ||G({\bm z}) - {\bm x}_i ||_{2},
\end{aligned}
\end{equation}
where $N$ is the size of sample set, $||\cdot ||_{2}$ denotes $\mathcal{L}_2$ norm, $G$ is the generator with input data ${\bm z}$ and ${\bm x}_i$ denotes a testing sample.

At last, for comparison, other GANs based on different objective functions, including the original GAN \cite{Generative-Adversarial-Nets}, LSGAN \cite{Least-squares-generative-adversarial-networks} and WGAN \cite{Wasserstein-GAN}, are also trained and tested in the same way as MIM-based GAN.
For different GANs, the data generation performance measured by RE is listed in the Table \ref{table.data_generation_ODDS}, where the statistic results of $20$ experiments are shown in the three datasets.

\subsubsection{Unsupervised anomaly detection for ODDS} \ \par

In order to evaluate the detection performance of the proposed method (mentioned in Section \ref{section.procedure_detection_GAN}), we do experiments based on the datasets described in Section \ref{section.datasets}.

Based on those imbalanced datasets from ODDS, several common criterions are exploited to show the detection performance, including Precision, Recall, $F_{1}$-score, Receiver Operating Characteristic (ROC) curve and Area Under Curve (AUC).
In particular, Precision, Recall, $F_{1}$-score and Accuracy are given by
\begin{subequations}
\begin{align}
& \text{Precision} = \frac{\text{TP}}{\text{TP}+\text{FP}}, \\
& \text{Recall} = \frac{\text{TP}}{\text{TP}+\text{FN}}, \\
& F_{1}\text{-score} =
\frac{2 \times \text{Recall} \times \text{Precision}}{\text{Recall}+\text{Precision}},\\
& \text{Accuracy} = \frac{\text{TP}+\text{TN}}{\text{TP}+\text{TN}+\text{FP}+\text{FN}},
\end{align}
\end{subequations}
which essentially depend on True Positive (TP), True Negative (TN), False Positive (FP) and False Negative (FN).

As for the anomaly detection method with MIM-based GAN mentioned in \ref{section.procedure_detection_GAN}, the key point lies in the training process for MIM-based GAN.
In this regard, the DNNs are chosen as the generator and discriminator, where the configurations including architectures, parameters, optimizers and so on, are the same as those in Section \ref{subsection.data_gen_ODDS}.
Actually, the training process is the same as that in the data generation experiments mentioned in Section \ref{subsection.data_gen_ODDS}.

After training, we use Eq. (\ref{eq.Lerror}) (with $\lambda=0.1$) to compute the anomaly score described in Eq. (\ref{eq.anomaly_score}) (with $\eta=0.05$), where $\lambda$ and $\eta$ are adjustable parameters used for the corresponding regularization terms in Eq. (\ref{eq.Lerror}) and Eq. (\ref{eq.anomaly_score}) respectively, as well as they do not make great effects on the results when getting small enough.
In the anomaly score, the optimal ${\bm z}_{\text{opt}}$ in latent space is achieved by means of the Adam optimizer with learning rate $0.003$.
As a result, the testing samples labeled by Eq. (\ref{eq.decision}) are classified into the anomalous or the normal.

Now, we compare the proposed method with other unsupervised anomaly detection methods in the three datasets from ODDS.
Since that in our proposed method, the main part of the framework mentioned in Fig. \ref{fig_detect_model} lies in MIM-based GAN, we can replace this part with different GANs (including the original GAN \cite{Generative-Adversarial-Nets}, LSGAN \cite{Least-squares-generative-adversarial-networks} and WGAN \cite{Wasserstein-GAN}) for comparison.
As well, several conventional methods such as K-Means algorithm via PCA \cite{K-means-Clustering-PCA,Dimensionality-Reduction}, Isolation Forest \cite{Isolation-forest,Isolation-based} and Minimum Covariance Determinant \cite{Outlier-detection}, are also compared with the above methods.

\subsection{Results and analysis}
\subsubsection{Data generation results for ODDS}\label{subsection.date_generation_ODDS}\ \par
From Table \ref{table.data_generation_ODDS}, it is not difficult to observe that different kinds of datasets make a difference in the data generation efficiency for these GANs.
However, generally speaking, there is not too much distinction on data generation capability for all these GANs when doing experiments on the same dataset.
\begin{table}[htbp]
\centering
\caption{Results of GAN-based data generation on ODDS}
\label{table.data_generation_ODDS}
\newcommand{\tabincell}[2]{\begin{tabular}{@{}#1@{}}#2\end{tabular}}
\setlength{\tabcolsep}{2.5mm}{
\begin{tabular}{|l|l|l|l|}
\bottomrule
\rowcolor{mygray}
\textbf{GAN Type} & \textbf{Cardio (RE)} &\textbf{Thyroid (RE)} &\textbf{Musk (RE)}\\
\hline
\tabincell{l}{Original GAN
\\ }
&
\tabincell{l}{$3.2464 \pm 0.0124$} & \tabincell{l}{$1.2242 \pm 0.0040$} & \tabincell{l}{$10.9774 \pm  0.0274$}\\
\hline
\tabincell{l}{LSGAN
\\}
&
\tabincell{l}{$2.9985 \pm 0.0114$} & \tabincell{l}{$1.2536 \pm 0.0057$} & \tabincell{l}{$10.0251 \pm 0.0270$}\\
\hline
\tabincell{l}{WGAN
\\}
&
\tabincell{l}{$3.4497 \pm 0.0097$} & \tabincell{l}{$1.4174 \pm 0.0131$} & \tabincell{l}{$9.1237 \pm 0.0433$}\\
\hline
\tabincell{l}{ {MIM-based GAN}
\\}
&
\tabincell{l}{$3.009 \pm 0.0101$} & \tabincell{l}{$1.0009 \pm 0.0023$} &
\tabincell{l}{$9.8383 \pm 0.0319$}\\
\toprule
\end{tabular}}
\end{table}

In particular, MIM-based GAN almost has superior performance of data generation on Cardio and Thyroid datasets, while WGAN has an advantage on Musk dataset.
It is also noteworthy that there exists larger variation on the data generation capability of WGAN, which seems to be influenced by the size and imbalance of different datasets.

In terms of the original GAN, it generally performs worse than MIM-based GAN, which may be caused by its model collapse in some degree.
As well, LSGAN hardly ever has more superiority in data generation than MIM-based GAN.

\begin{table}[htbp]
\centering
\caption{Results of unsupervised anomaly detection on Cardio dataset}
\label{table.Cardio}
\newcommand{\tabincell}[2]{\begin{tabular}{@{}#1@{}}#2\end{tabular}}
\setlength{\tabcolsep}{2.5mm}{
\begin{tabular}{|l|l|l|l|l|}
\bottomrule
\rowcolor{mygray}
\textbf{Methods} & \textbf{Precision} &\textbf{Recall} &\textbf{$F_1$-score} &\textbf{Accuracy}\\
\hline
\tabincell{l}{K-Means$+$PCA\\ }
&
\tabincell{l}{$91.03\% \pm 1.74\%$} & \tabincell{l}{$93.14\% \pm 4.61\%$} & \tabincell{l}{$91.98\% \pm  1.46\%$} & \tabincell{l}{$85.50\% \pm 2.16\%$}\\
\hline
\tabincell{l}{Isolation Forest\\ }
&
\tabincell{l}{$93.45\% \pm 0.39\%$} & \tabincell{l}{$95.34\% \pm 0.56\%$} & \tabincell{l}{$94.39\% \pm  0.24\%$} & \tabincell{l}{$89.82\% \pm 0.41\%$}\\
\hline
\tabincell{l}{Minimum Covariance Determinant \\ }
&
\tabincell{l}{$92.44\% \pm 0.76\%$} & \tabincell{l}{$95.24\% \pm 0.33\%$} & \tabincell{l}{$93.82\% \pm  0.46\%$} & \tabincell{l}{$88.73\% \pm 0.88\%$}\\
\hline
\tabincell{l}{Original GAN Method \\ }
&
\tabincell{l}{$95.00\% \pm 0.80\%$} & \tabincell{l}{$81.11\% \pm 4.12\%$} & \tabincell{l}{$87.47\% \pm  2.75\%$} & \tabincell{l}{$79.23\% \pm 3.94\%$}\\
\hline
\tabincell{l}{LSGAN Method\\}
&
\tabincell{l}{$95.97\% \pm 0.46\%$} & \tabincell{l}{$86.23\% \pm 1.49\%$} & \tabincell{l}{$90.83\% \pm 0.93\%$} & \tabincell{l}{$84.39\% \pm 1.47\%$}\\
\hline
\tabincell{l}{WGAN Method\\}
&
\tabincell{l}{$95.91\% \pm 2.21\%$} & \tabincell{l}{$81.42\% \pm 11.25\%$} & \tabincell{l}{$87.68\% \pm 6.93\%$} & \tabincell{l}{$80.18\% \pm 9.99\%$}\\
\hline
\tabincell{l}{{MIM-based} {GAN} {Method}\\}
&
\tabincell{l}{$96.31\% \pm 0.59\%$} & \tabincell{l}{$86.99\% \pm 2.10\%$} & \tabincell{l}{$91.40\% \pm 1.15\%$} & \tabincell{l}{$85.32\% \pm 1.79\%$}\\
\toprule
\end{tabular}}
\end{table}
\begin{table}[htbp]
\centering
\caption{Results of unsupervised anomaly detection on Thyroid dataset}
\label{table.Thyroid}
\newcommand{\tabincell}[2]{\begin{tabular}{@{}#1@{}}#2\end{tabular}}
\setlength{\tabcolsep}{2.5mm}{
\begin{tabular}{|l|l|l|l|l|}
\bottomrule
\rowcolor{mygray}
\textbf{Methods} & \textbf{Precision} &\textbf{Recall} &\textbf{$F_1$-score} &\textbf{Accuracy}\\
\hline
\tabincell{l}{K-Means$+$PCA\\ }
&
\tabincell{l}{$97.71\% \pm 0.0024\% $} & \tabincell{l}{$95.60\% \pm 0.10\%$} & \tabincell{l}{$96.65\% \pm 0.05\%$} & \tabincell{l}{$93.51\% \pm 0.10\%$}\\
\hline
\tabincell{l}{Isolation Forest\\ }
&
\tabincell{l}{$99.73\% \pm 0.11\%$} & \tabincell{l}{$91.91\% \pm 0.78\%$} & \tabincell{l}{$95.66\% \pm  0.43\%$} & \tabincell{l}{$91.84\% \pm 0.78\%$}\\
\hline
\tabincell{l}{Minimum Covariance Determinant \\ }
&
\tabincell{l}{$99.65\% \pm 0.0002 \%$} & \tabincell{l}{$91.67\% \pm 0.04\%$} & \tabincell{l}{$95.50\% \pm 0.02\%$} & \tabincell{l}{$91.54\% \pm 0.04\%$}\\
\hline
\tabincell{l}{Original GAN Method \\ }
&
\tabincell{l}{$98.87\% \pm 0.30\%$} & \tabincell{l}{$93.04\% \pm 1.11\%$} & \tabincell{l}{$95.86\% \pm  0.63\%$} & \tabincell{l}{$92.15\% \pm 1.15\%$}\\
\hline
\tabincell{l}{LSGAN Method\\}
&
\tabincell{l}{$98.97\% \pm 0.41\%$} & \tabincell{l}{$91.46\% \pm 1.78\%$} & \tabincell{l}{$95.06\% \pm 0.82\%$} & \tabincell{l}{$90.71\% \pm 1.46\%$}\\
\hline
\tabincell{l}{WGAN Method\\}
&
\tabincell{l}{$99.13\% \pm 0.72\%$} & \tabincell{l}{$91.37\% \pm 2.32\%$} & \tabincell{l}{$95.08\% \pm 1.37\%$} & \tabincell{l}{$90.78\% \pm 2.47\%$}\\
\hline
\tabincell{l}{{MIM-based} {GAN} {Method}\\}
&
\tabincell{l}{$99.53\% \pm 0.41\%$} & \tabincell{l}{$92.86\% \pm 0.43\%$} & \tabincell{l}{$96.08\% \pm 0.28\%$} & \tabincell{l}{$92.59\% \pm 0.53\%$}\\
\toprule
\end{tabular}}
\end{table}
\begin{table}[htbp]
\centering
\caption{Results of unsupervised anomaly detection on Musk dataset}
\label{table.Musk}
\newcommand{\tabincell}[2]{\begin{tabular}{@{}#1@{}}#2\end{tabular}}
\setlength{\tabcolsep}{2.5mm}{
\begin{tabular}{|l|l|l|l|l|}
\bottomrule
\rowcolor{mygray}
\textbf{Methods} & \textbf{Precision} &\textbf{Recall} &\textbf{$F_1$-score} &\textbf{Accuracy} \\
\hline
\tabincell{l}{K-Means$+$PCA\\ }
&
\tabincell{l}{$98.02\% \pm 0.77\%$} & \tabincell{l}{$91.91\% \pm 9.61\%$} & \tabincell{l}{$94.59\% \pm  5.12\%$} & \tabincell{l}{$90.24\% \pm 8.58\%$} \\
\hline
\tabincell{l}{Isolation Forest\\ }
&
\tabincell{l}{$100.00\% \pm 0.00\%$} & \tabincell{l}{$92.50\% \pm 2.67\%$} & \tabincell{l}{$96.09\% \pm  1.45\%$} & \tabincell{l}{$92.69\% \pm 2.61\%$} \\
\hline
\tabincell{l}{Minimum Covariance Determinant \\}
&
\tabincell{l}{$100.00\% \pm 0.00\%$} & \tabincell{l}{$93.76\% \pm 0.46\%$} & \tabincell{l}{$96.78\% \pm  0.25\%$} & \tabincell{l}{$93.91\% \pm 0.45\%$} \\
\hline
\tabincell{l}{Original GAN Method \\ }
&
\tabincell{l}{$100.00\% \pm 0.00\%$} & \tabincell{l}{$94.90\% \pm 2.10\%$} & \tabincell{l}{$97.37\% \pm  1.10\%$} & \tabincell{l}{$95.02\% \pm 2.04\%$} \\
\hline
\tabincell{l}{LSGAN Method\\}
&
\tabincell{l}{$100.00\% \pm 0.00\%$} & \tabincell{l}{$98.40\% \pm 0.53\%$} & \tabincell{l}{$99.19\% \pm 0.27\%$} & \tabincell{l}{$98.44\% \pm 0.51\%$} \\
\hline
\tabincell{l}{WGAN Method\\}
&
\tabincell{l}{$99.63\% \pm 0.45\%$} & \tabincell{l}{$98.21\% \pm 0.28\%$} & \tabincell{l}{$98.91\% \pm 0.20\%$} & \tabincell{l}{$97.89\% \pm 0.38\%$} \\
\hline
\tabincell{l}{{MIM-based} {GAN} {Method}\\}
&
\tabincell{l}{$100.00\% \pm 0.00\%$} & \tabincell{l}{$99.28\% \pm 0.35\%$} & \tabincell{l}{$99.64\% \pm 0.18\%$} & \tabincell{l}{$99.30\% \pm 0.34\%$} \\
\toprule
\end{tabular}}
\end{table}

\begin{figure}[!hbt]
\centering
\includegraphics[width=5.7 in]{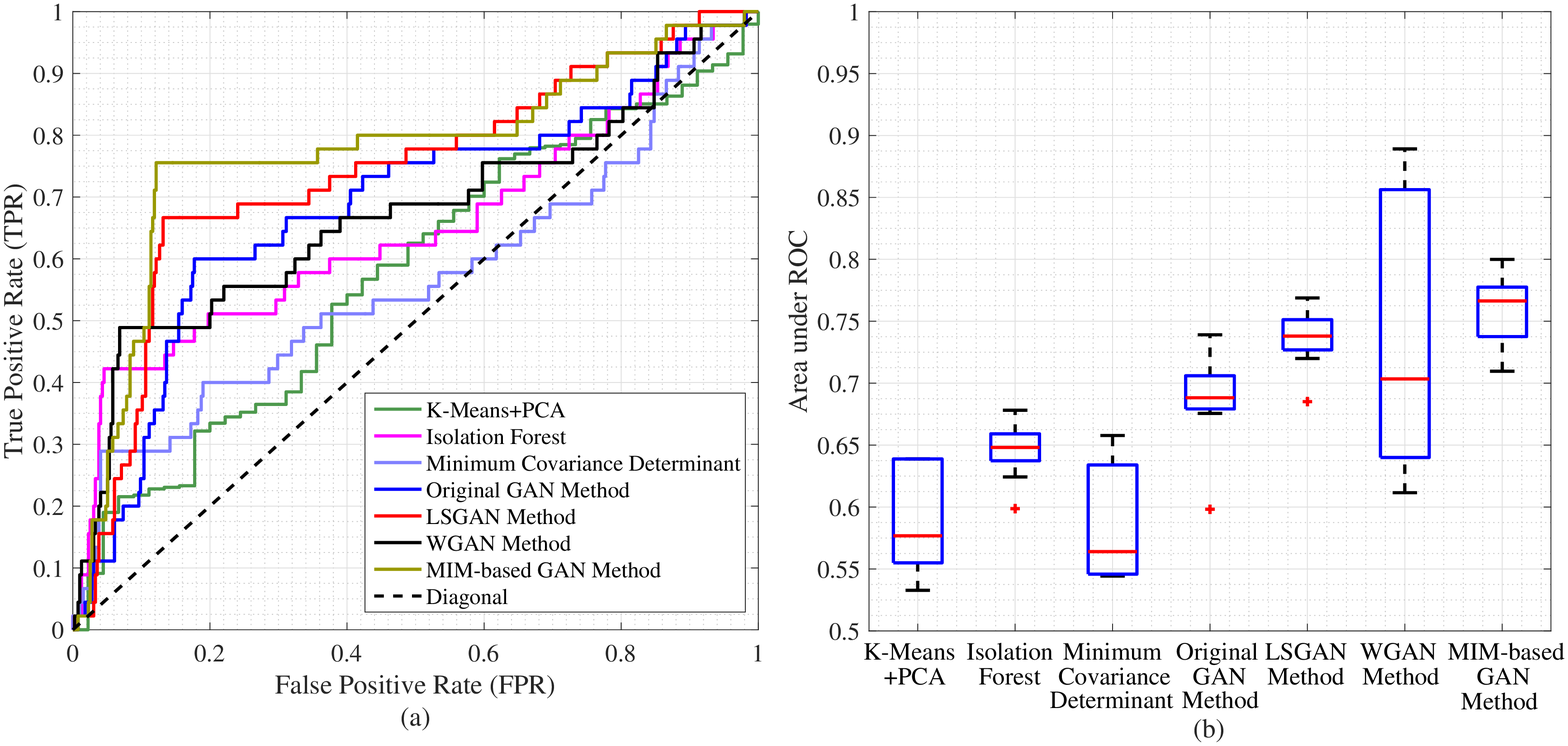}
\caption{
Considering the unsupervised anomaly detection for the Cardio dataset, after randomly shuffling the data, $1,360$ samples (about $74.28 \%$ data) are used as training data and the rest are considered as the testing data.
Especially, with respect to the GAN-based methods, we use the original GAN, WGAN, LSGAN and MIM-based GAN to produce the required generative data, which correspond to KL method, W method, LS method and MIM method, respectively, where the training iterations is $1000$.
To show the performance on anomaly detection for each method, the ROC curve for the experiment with the median AUC as well as the boxplot of AUC for all the experiments are drawn as subfigures ($a$) and ($b$) respectively, where there are $20$ experiments.
}
\label{fig_cardio}
\end{figure}

\subsubsection{Unsupervised anomaly detection results for ODDS}\label{section.ad_result_ODDS}\ \par
TABLE \ref{table.Cardio}, \ref{table.Thyroid} and \ref{table.Musk} evaluate the anomaly detection performance for several methods with the statistics (including mean value and standard deviation) during $20$ experiments.
From Fig. \ref{fig_cardio}, \ref{fig_thyroid} and \ref{fig_musk}, 
ROC curve and AUC are also drawn to provide the comparison for the unsupervised anomaly detection methods.
Since the three datasets of ODDS have different characteristics,
there exists different performance for these discussed methods.
While, generally speaking, the proposed MIM-based GAN method holds its superiority on the anomaly detection in the case of higher-dimensional data and smaller proportion of anomalies.

Specifically, according to the performance on Cardio dataset from TABLE \ref{table.Cardio} and Fig. \ref{fig_cardio}, it is readily seen that the proposed MIM-based GAN method performs better than other methods in terms of Precision, ROC curve and AUC.
While, with respect to Recall, $F_1$-score and Accuracy, it has worse performance than the conventional methods, especially Isolation Forest, though it is better than the other GAN-based methods.
The reason may lie in not small enough proportion of anomalous samples and not enough training samples in Cardio dataset.

With respect to the performance on Thyroid dataset from TABLE \ref{table.Thyroid} and Fig. \ref{fig_thyroid}, the MIM-based GAN method almost plays the suboptimal role compared with other methods.
Actually, due to this dataset without too high dimension, it is reasonable that the GAN-based methods do not have much more advantages than the conventional methods.

\begin{figure}[!hbt]
\centering
\includegraphics[width=5.7 in]{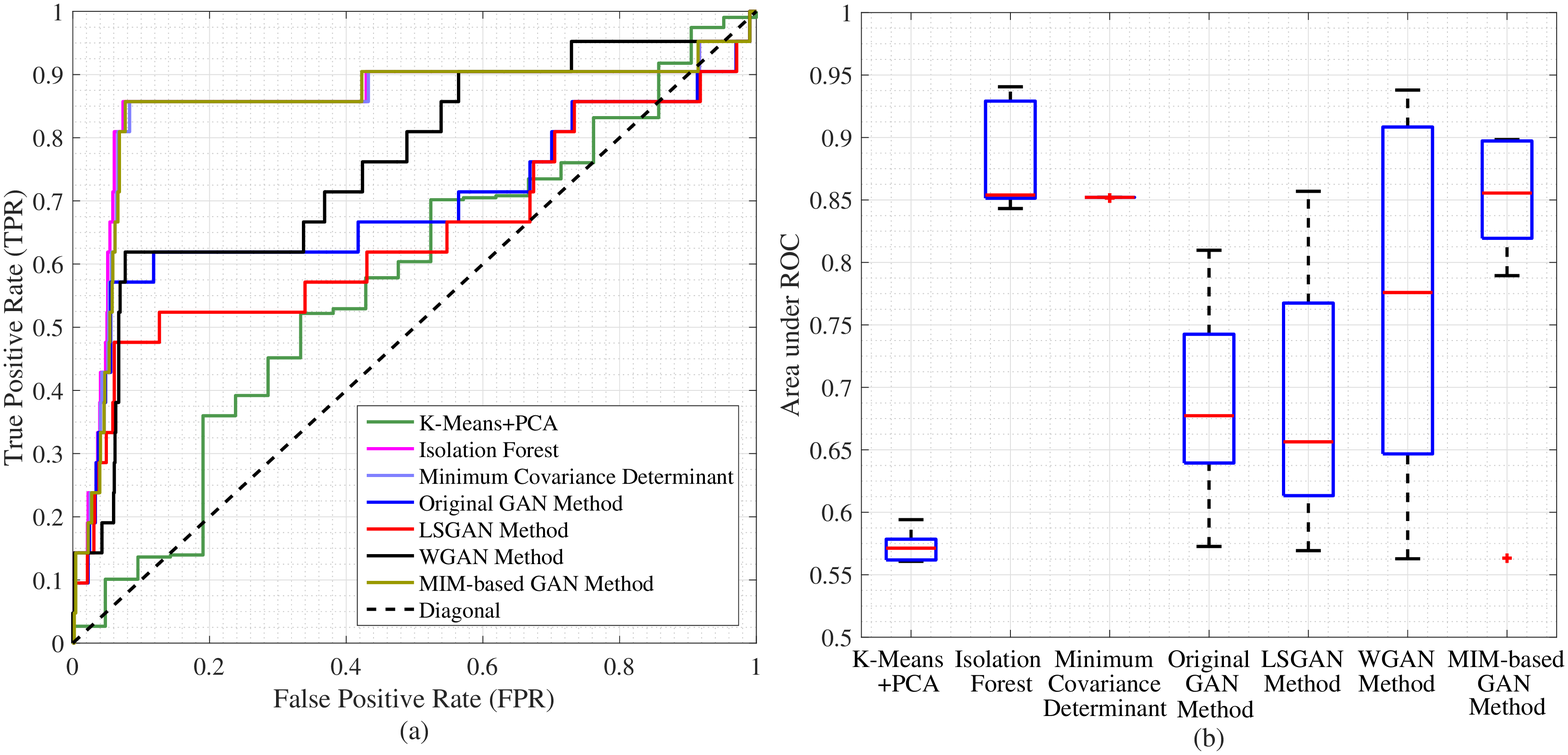}
\caption{
As for the unsupvervised anomaly detection based on Thyroid dataset, $2800$ samples (about $74.23 \%$ data) are used as training data. Besides, the rest samples are regarded as the testing data.
Especially, there are $1500$ training iterations for the GAN-based methods.
To intuitively show the detection result for each method, the ROC curve for the experiment with the median AUC as well as the boxplot of AUC for all the experiments are drawn (as subfigures ($a$) and ($b$)), where $20$ experiments are taken.
}
\label{fig_thyroid}
\end{figure}
\begin{figure}[!hbt]
\centering
\includegraphics[width=5.7 in]{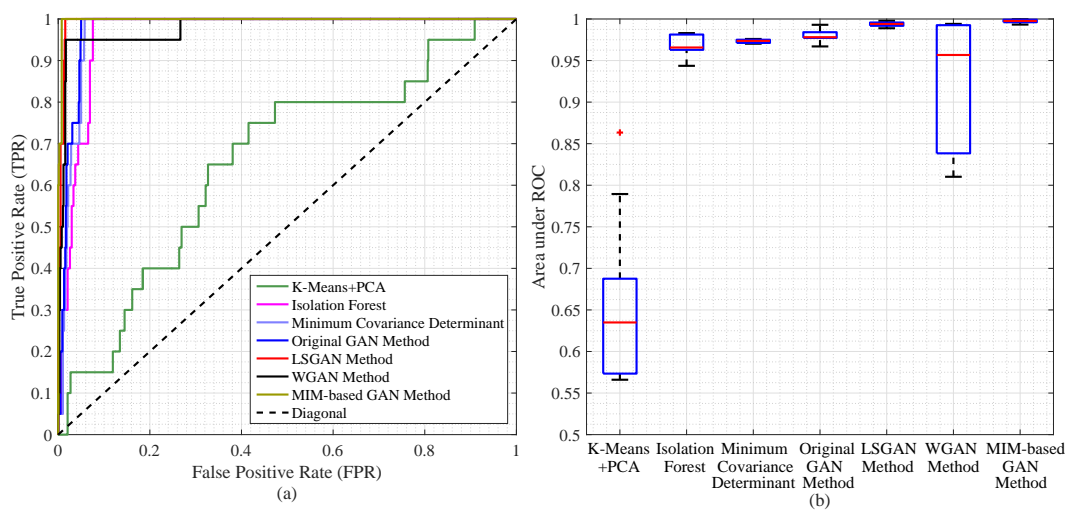}
\caption{
In terms of the unsupervised anomaly detection for the Musk dataset, $2,200$ samples (about $71.85 \%$ data) are used as training data, while the rest samples are used as testing data.
Moreover, the number of training iterations for the GAN-based methods is $50$.
In order to intuitively show the different detection results, with respect to each method, the ROC curve for the experiment with the median AUC is shown as subfigure ($a$), as well as, the boxplot of AUC for all the experiments is drawn as subfigure ($b$), where the performance is evaluated on $20$ experiments.
}
\label{fig_musk}
\end{figure}

When taking experiments on Musk dataset,
the MIM-based GAN method holds the optimal performance compared with the other methods, according to TABLE \ref{table.Musk} and Fig. \ref{fig_musk}.
This also implies that MIM-based GAN method have its advantage on unsupervised anomaly detection for the dataset with high enough dimension and a small enough amount of anomalies.

\subsubsection{Analysis for experiments}\ \par

As for the data generation results in Section \ref{subsection.date_generation_ODDS},
MIM-based GAN is compared with other conventional GANs, showing generally optimal or suboptimal performance.
In some degree, this results from its superior characteristics discussed in Section \ref{subsection.MIM-GAN}, including mode collapse resistance, anti-interference capability of generator gradient and so on.
Moreover, in terms of WGAN, its data generation capability is not robust enough in different datasets,
since its linear form of the objective function is relatively sensitive to data varying attributes.
As for the original GAN and LSGAN, they almost have worse performance than MIM-based GAN, which may  respectively be attributed to model collapse and small probability event ignorance in some degree, according to Section \ref{subsection.MIM-GAN}.

As for the unsupervised anomaly detection results in Section \ref{section.ad_result_ODDS},
it is noteworthy that the MIM-based GAN method almost performs better than the other discussed GAN-based methods (including the original GAN method, LSGAN method and WGAN method) in the discussed datasets.
This is not only related to the data generation capability of MIM-based GAN, but also due to its objective function, in which the proportion of small probability events is amplified by the corresponding information metric, as discussed in Section \ref{section.rare_events_generator}.

Specifically, as for the original GAN and LSGAN, they do not pay more attention on small probability events from the perspective of objective function, while MIM-based GAN does, as well as it also has better theoretical characteristics mentioned in Section \ref{subsection.MIM-GAN}.
As for WGAN, its uncertainty of small probability event proportion caused by the linear objective function may make it more unstable to anomaly detection than MIM-based GAN.

In addition, dataset attributes, including the dimension of features, the number of samples and the proportion of anomalies, have great impacts on the anomaly detection performance.
Compared with the conventional unsupervised anomaly detection methods,
GAN-based methods are better at dealing with higher-dimensional data in the case of adequate training samples with less proportion of anomalies.

\section{Conclusion}
In this paper, a new approach deemed MIM-based GAN is introduced as an alternative to the conventional GANs which are used in anomaly detection.
In this model, the information metric based on the exponential function is exploited in the corresponding objective function, from which several novel characteristics are gained in theory, especially for small probability event influence enhancement.
Furthermore, an unsupervised anomaly detection method with MIM-based GAN is presented, whose principle is also explained from the perspective of event probability influence on data generation.
At last, several online datasets are used to evaluate the data generation of MIM-based GAN and compare its corresponding anomaly detection method with other methods.

\appendices
\section{Proof of Lemma \ref{lem.Optimality}}\label{app.lem.Optimality}
For a given generator $g_{\bm \theta}$, the training criterion of discriminator $D$ in MIM-based GAN is to minimize
\begin{equation}
\begin{aligned}
    & L_{\text{MIM}}(D,G)
    = \int_{\mathcal{X}}[ P({\bm x})\exp(1-D({\bm x})) + P_{g_{\theta}}({\bm x})\exp(D({\bm x}))] {\rm d}{\bm x}\text{,}
\end{aligned}
\end{equation}
where the optimal discriminator $D^{*}_{\text{\rm MIM}}$ is the solution, as well as $\mathcal{X}$ is the domain of ${\bm x}$.

Supposing that there is a function $F(u)=a \exp(1-u)+b\exp(u)$ $(a>0, b>0)$,
we have the solution $u=\frac{1}{2}+ \frac{1}{2}\ln(\frac{a}{b})$ which satisfies $\frac{\partial F(u)}{\partial u}=-a\exp(1-u)+b\exp(u)=0$.
Since that the second order derivative satisfies
$\frac{\partial^2 F(u)}{\partial u^2}=a\exp(1-u)+b\exp(u)>0$, it implies that $F(u)$ is convex with respect to $u$. The solution $u=\frac{1}{2}+\frac{1}{2}\ln(\frac{a}{b})$ is readily obtained to achieve the minimum value of $F(u)$, from which this Lemma has been already proved.

\section{Proof of Proposition \ref{prop.maximum}}\label{app.prop.maximum}
If it is satisfied that $\mathbb{P}=\mathbb{P}_{g_{\theta}}$ namely $D^{*}_{\text{MIM}}({\bm x})=\frac{1}{2}$, we have the value of Eq. (\ref{eq.L_Doptimal}) as
$L_{\text{MIM}}(D=\frac{1}{2},G)=\sqrt{{\rm e}}(1+1)=2\sqrt{{\rm e}}$.
In other words, the maximum value of $L_{\text{MIM}}(D=D^{*}_{\text{MIM}},G)$ is achieved at the point $\mathbb{P}=\mathbb{P}_{g_{\theta}}$.
Moreover, the equivalent formulation for Eq. (\ref{eq.L_Doptimal}) is given by
\begin{equation}
\begin{aligned}
    & \max_{g_{\bm\theta}} L_{\text{MIM}}(D=D^{*}_{\text{MIM}},G)
    \Leftrightarrow \max_{g_{\bm\theta}}
    \sqrt{{\rm e}} \bigg\{
     \ln \mathbb{E}_{{\bm x}\sim \mathbb{P}}
     \bigg[\bigg(\frac{P({\bm x})}{P_{g_{\theta}}({\bm x})}\bigg)^{-\frac{1}{2}}\bigg]
    + \ln \mathbb{E}_{{\bm x}\sim \mathbb{P}_{g_{\theta}}}
    \bigg[\bigg(\frac{P_{g_{\theta}}({\bm x})}{P({\bm x})}\bigg)^{-\frac{1}{2}}\bigg]
    \bigg\}\text{.}
\end{aligned}
\end{equation}
Then, it is readily seen that
\begin{equation}
\begin{aligned}
    & \max_{g_{\bm\theta}} L_{\text{MIM}}(D=D^{*}_{\text{MIM}},G)
    \Leftrightarrow \min_{g_{\bm\theta}}
    \frac{\sqrt{\rm e}}{2}
    \bigg\{ R_{\alpha=\frac{1}{2}}(\mathbb{P}||\mathbb{P}_{g_{\theta}})
        + R_{\alpha=\frac{1}{2}}(\mathbb{P}_{g_{\theta}}||\mathbb{P})  \bigg\}\text{,}
\end{aligned}
\end{equation}
where $R_{\alpha=\frac{1}{2}}(\cdot)$ is the Renyi divergence (with the parameter $\alpha=\frac{1}{2}$) whose definition is given by
\begin{equation}
\begin{aligned}
    & R_{\alpha}(\mathbb{P}||\mathbb{P}_{g_{\theta}})
    =\frac{1}{\alpha-1} \ln \bigg\{\mathbb{E}_{{\bm x}\sim \mathbb{P}} \bigg[\bigg(\frac{P({\bm x})}{P_{g_{\theta}}({\bm x})}\bigg)^{\alpha-1}\bigg] \bigg\}\text{,}
    (\alpha >0, \alpha \ne 1)\text{.}
\end{aligned}
\end{equation}
Since that Renyi divergence achieves the minimum under the condition of two equal distributions, it is easy to see that $2\sqrt{{\rm e}}$ is gotten as the global maximum of $L_{\text{MIM}}(D=D^{*}_{\text{MIM}},G)$ in the case that $\mathbb{P}=\mathbb{P}_{g_{\theta}}$ (which implies the generative model replicates the real data).

\section{Proof of Proposition \ref{prop.mode_collapse}}\label{app.prop.mode_collapse}

When the optimal discriminator is achieved, the data generation of MIM-based GAN depends on the second term of Eq. (\ref{eq.LMIM}), for which the optimization problem can be described as
\begin{equation}\label{eq.MaxG_DMIM}
\begin{aligned}
    & \max_{G} \mathbb{E}_{{\bm z}\sim \mathbb{P}_{z}}[\exp(D^{*}_{\text{\rm MIM}}(G({\bm z})))]\\
    & \Leftrightarrow \max_{\bm \theta} \mathbb{E}_{{\bm x}\sim \mathbb{P}_{g_{\theta}}}[\exp(D^{*}_{\text{\rm MIM}}({\bm x}))]\\
    & \Leftrightarrow \max_{\bm \theta} \mathbb{E}_{{\bm x}\sim \mathbb{P}_{g_{\theta}}}
    \bigg[\exp\bigg( \frac{1}{2}+\frac{1}{2}\ln \frac{P({\bm x})}{P_{g_{\theta}}({\bm x})} \bigg) \bigg]\\
    & \Leftrightarrow \max_{\bm \theta} \sqrt{{\rm e}} \bigg\{
    \mathbb{E}_{{\bm x}\sim \mathbb{P}_{g_{\theta}}} \bigg[ \bigg(\frac{P({\bm x})}{P_{g_{\theta}}({\bm x})}\bigg)^{\frac{1}{2}} \bigg]
    \bigg\}\\
    & \Leftrightarrow \max_{\bm \theta} \int_{\mathcal{X}} P_{g_{\theta}}({\bm x})
    \bigg(\frac{P({\bm x})}{P_{g_{\theta}}({\bm x})}\bigg)^{\frac{1}{2}} {\rm d}{\bm x}\\
    & \Leftrightarrow \max_{\bm \theta}
    \int_{\mathcal{X}} {P_{g_{\theta}}({\bm x})}^{\frac{1}{2}} {P({\bm x})}^{\frac{1}{2}} {\rm d}{\bm x},
\end{aligned}
\end{equation}
where the optimal discriminator $D^{*}_{\text{\rm MIM}}$ is given by Eq. (\ref{eq.D*}) and $\mathcal{X}$ is the domain of ${\bm x}$.

Similarly, as for the original GAN with its optimal discriminator (i.e. $D^{*}_{\text{\rm KL}}({\bm x})=\frac{P({\bm x})}{P({\bm x})+P_{g_{\theta}}({\bm x})}$), the optimization problem for data generation is given by
\begin{equation}\label{eq.MinG_DKL}
\begin{aligned}
    & \min_{G} \mathbb{E}_{{\bm z}\sim \mathbb{P}_{z}}[\ln(1-D^{*}_{\text{\rm KL}}(G({\bm z})))] \\
    & \Leftrightarrow \min_{\bm \theta} \mathbb{E}_{{\bm x}\sim \mathbb{P}_{g_{\theta}}}[\ln(1-D^{*}_{\text{\rm KL}}({\bm x}))] \\
    & \Leftrightarrow \min_{\bm \theta} \mathbb{E}_{{\bm x}\sim \mathbb{P}_{g_{\theta}}}
    \bigg[\ln\bigg( \frac{P_{g_{\theta}}({\bm x})}{P({\bm x})+P_{g_{\theta}}({\bm x})} \bigg)\bigg] \\
    & \Leftrightarrow \max_{\bm \theta} \mathbb{E}_{{\bm x}\sim \mathbb{P}_{g_{\theta}}}
    \bigg[\ln\bigg( 1+\frac{P({\bm x})}{P_{g_{\theta}}({\bm x})} \bigg)\bigg]\\
    & \Leftrightarrow \max_{\bm \theta} \mathbb{E}_{{\bm x}\sim \mathbb{P}_{g_{\theta}}}
    \bigg[\ln\bigg(\frac{P({\bm x})}{P_{g_{\theta}}({\bm x})} \bigg)\bigg]\\
    & \Leftrightarrow \min_{\bm \theta}  \int_{\mathcal{X}} {P_{g_{\theta}}({\bm x})}
    \bigg[\ln\bigg(\frac{P_{g_{\theta}}({\bm x})}{P({\bm x})} \bigg)\bigg] {\rm d}{\bm x}, \\
\end{aligned}
\end{equation}
which implies that this optimization problem is equal to minimizing the KLD between $P_{g_{\theta}}({\bm x})$ and $P({\bm x})$.

In terms of LSGAN mentioned in Table \ref{table.summary_GANs}, its optimal discriminator is given by
$D^{*}_{\text{\rm LS}}({\bm x})=\frac{P({\bm x})}{P({\bm x})+P_{g_{\theta}}({\bm x})}$.
Then, the optimization problem for data generation is described as
\begin{equation}\label{eq.MinG_DLS}
\begin{aligned}
    & \max_{G} \frac{1}{2} \mathbb{E}_{{\bm z}\sim \mathbb{P}_{z}}[(D^{*}_{\text{\rm LS}}(G({\bm z})))^2] \\
    & \Leftrightarrow \max_{\bm \theta} \mathbb{E}_{{\bm x}\sim \mathbb{P}_{g_{\theta}}}
    \bigg[\bigg( \frac{P({\bm x})}{P({\bm x})+P_{g_{\theta}}({\bm x})} \bigg)^2\bigg] \\
    & \Leftrightarrow \max_{\bm \theta}  \int_{\mathcal{X}} {P_{g_{\theta}}({\bm x})}
    \bigg[\bigg(\frac{P({\bm x})}{P({\bm x})+P_{g_{\theta}}({\bm x})} \bigg)^2\bigg] {\rm d}{\bm x}\\
    & \Leftrightarrow \max_{\bm \theta}  \int_{\mathcal{X}} {P({\bm x})}
    \bigg[\frac{P({\bm x})P_{g_{\theta}}({\bm x})}{(P({\bm x})+P_{g_{\theta}}({\bm x}))^2} \bigg] {\rm d}{\bm x},\\
\end{aligned}
\end{equation}
whose optimal solution is reached at $\mathbb{P}=\mathbb{P}_{g_{\theta}}$.
This results from
${P({\bm x})P_{g_{\theta}}({\bm x})} \le \frac{(P({\bm x})+P_{g_{\theta}}({\bm x}))^2}{2}$ where the equality is held if and only if $P({\bm x})=P_{g_{\theta}}({\bm x})$.

According to Eq. (\ref{eq.MinG_DKL}), it is not difficult to see that for the original GAN, when $P_{g_{\theta}}({\bm x}) \to 0$, the minimum is held with any $P({\bm x})$.
This implies that the objective function in the original GAN may not tend to generate more efficient data categories for the real distribution.
However, in the light of Eq. (\ref{eq.MaxG_DMIM}) and Eq. (\ref{eq.MinG_DLS}), the optimization problems for MIM-based GAN and LSGAN avoid $P_{g_{\theta}}({\bm x}) \to 0$, which refrains from the model collapse.

Considering the objective function of WGAN in Table \ref{table.summary_GANs}, it is not difficult to see that the optimal discriminator $D^{*}_{\text{\rm W}}({\bm x})$ and its corresponding optimization problem of data generation depend on $\{P({\bm x})-P_{g_{\theta}}({\bm x})\}$, where
$P({\bm x})$ and $P_{g_{\theta}}({\bm x})$ denote the densities of real distribution and generative distribution, respectively.
In other words, the optimization problem just encourages $\{P({\bm x})-P_{g_{\theta}}({\bm x})\} \to 0$ rather than $P_{g_{\theta}}({\bm x}) \to 0$, which implies WGAN also mitigates the model collapse.

\section{Proof of Proposition \ref{prop.stability}}\label{app.prop.stability}
Consider the fact that the gradient of objective functions with respect to the parameter of generator does not depend on the parameter of discriminator.
In this regard, it is reasonable that the gradient of generator of MIM-based GAN is represented by $\nabla_{\bm\theta}\mathbb{E}_{{\bm z}\sim \mathbb{P}_z} [\exp(D(g_{\bm\theta}({\bm z}))) ]$.

Based on the assumptions mentioned in Proposition \ref{prop.stability}, it is not difficult to see that
\begin{equation}\label{eq.gradient_D*}
\begin{aligned}
    & \nabla_{\bm\theta}\mathbb{E}_{{\bm z}\sim \mathbb{P}_z} [\exp({\tilde D^*}(g_{\bm\theta}({\bm z}))) ]
    = \mathbb{E}_{{\bm z}\sim \mathbb{P}_z} [\exp({\tilde D^*}(g_{\bm\theta}({\bm z})))
    \nabla_{{\bm x}}{\tilde D^*}({\bm x})\nabla_{\bm\theta}g_{\bm\theta}({\bm z}) ]
    =0,
\end{aligned}
\end{equation}
as well as
\begin{equation}\label{eq.gradient_D}
\begin{aligned}
    & \nabla_{\bm\theta}\mathbb{E}_{{\bm z}\sim \mathbb{P}_z} [\exp(D(g_{\bm\theta}({\bm z}))) ]\\
    & = \mathbb{E}_{{\bm z}\sim \mathbb{P}_z} [\exp(D(g_{\bm\theta}({\bm z})))
    \nabla_{{\bm x}}D({\bm x})\nabla_{\bm\theta}g_{\bm\theta}({\bm z}) ]\\
    & = \mathbb{E}_{{\bm z}\sim \mathbb{P}_z} [\exp({\tilde D^*}(g_{\bm\theta}({\bm z}))+\delta) (\nabla_{{\bm x}}{\tilde D^*}({\bm x}) + \upsilon)\nabla_{\bm\theta}g_{\bm\theta}({\bm z}) ]\\
    & = \exp(\delta) \upsilon \mathbb{E}_{{\bm z}\sim \mathbb{P}_z}
    [\nabla_{\bm\theta}g_{\bm\theta}({\bm z}) ] \text{, }
\end{aligned}
\end{equation}
where the notations are the same as those in Proposition \ref{prop.stability}.
Comparing Eq. (\ref{eq.gradient_D*}) and Eq. (\ref{eq.gradient_D}), this proposition is verified.

\section{Proof of Corollary \ref{cor.stability}}\label{app.cor.stability}
As for MIM-based GAN and the original GAN, it is readily seen that
$\nabla_{\bm\theta} L_{\text {\rm MIM}}(D,g_{\bm\theta})$ and $\nabla_{\bm\theta} L_{\text {\rm KL}}(D,g_{\bm\theta})$ depend on
$\nabla_{\bm\theta}\mathbb{E}_{{\bm z}\sim \mathbb{P}_z} [\exp(D(g_{\bm\theta}({\bm z}))) ]$ and
$\nabla_{\bm\theta}\mathbb{E}_{{\bm z}\sim \mathbb{P}_z} [\ln (1-D(g_{\bm\theta}({\bm z}))) ]$, respectively.

When there exists the condition satisfying $D-\tilde D^*=\delta$ and $\nabla_{{\bm x}} D({\bm x})-\nabla_{{\bm x}} \tilde D^*({\bm x})=\upsilon$ ($\tilde D^*(g_{\bm\theta}({\bm z}))=0$ and $\nabla_{{\bm x}} \tilde D^*({\bm x})=0$), we have
\begin{equation}\label{eq.KL_gradient_D*}
\begin{aligned}
    & \nabla_{\bm\theta}\mathbb{E}_{{\bm z}\sim \mathbb{P}_z} [\ln (1-{\tilde D^*}(g_{\bm\theta}({\bm z}))) ]
    = \mathbb{E}_{{\bm z}\sim \mathbb{P}_z}
    [- \frac{ \nabla_{{\bm x}}{\tilde D^*}({\bm x}) \nabla_{\bm\theta}g_{\bm\theta}({\bm z})}
    {1-\tilde D^*(g_{\bm\theta}({\bm z}))} ]
     = 0 \text{,}
\end{aligned}
\end{equation}
as well as
\begin{equation}\label{eq.KL1_gradient_D}
\begin{aligned}
    & \nabla_{\bm\theta}\mathbb{E}_{{\bm z}\sim \mathbb{P}_z} [\ln (1-D(g_{\bm\theta}({\bm z}))) ]\\
    & = \mathbb{E}_{{\bm z}\sim \mathbb{P}_z}
    [- \frac{ \nabla_{{\bm x}}D({\bm x})\nabla_{\bm\theta}g_{\bm\theta}({\bm z})}
    {1-D(g_{\bm\theta}({\bm z}))} ]\\
    & = \mathbb{E}_{{\bm z}\sim \mathbb{P}_z}
    [- \frac{ (\nabla_{{\bm x}}{\tilde D^*}({\bm x})+ \upsilon) \nabla_{\bm\theta}g_{\bm\theta}({\bm z})}
    {1-\tilde D^*(g_{\bm\theta}({\bm z})) -\delta} ]\\
    & = -\frac{\upsilon}{1 -\delta}
    \mathbb{E}_{{\bm z}\sim \mathbb{P}_z}[\nabla_{\bm\theta}g_{\bm\theta}({\bm z})]\text{.}
\end{aligned}
\end{equation}

Then, it is readily seen that
\begin{equation}\label{eq.KL1_perfect_stability}
\begin{aligned}
    & \nabla_{\bm\theta}\mathbb{E}_{{\bm z}\sim \mathbb{P}_z} [\ln (1-D(g_{\bm\theta}({\bm z}))) ]
    - \nabla_{\bm\theta}\mathbb{E}_{{\bm z}\sim \mathbb{P}_z} [\ln (1-{\tilde D^*}(g_{\bm\theta}({\bm z}))) ]
    = -\frac{\upsilon}{1 -\delta}
    \mathbb{E}_{{\bm z}\sim \mathbb{P}_z}[\nabla_{\bm\theta}g_{\bm\theta}({\bm z})]\text{.}
\end{aligned}
\end{equation}

Assume there is a function $h(u)=\frac{1}{1-u}-\exp(u)$ ($u\in (0,1)$).
Due to the fact $\exp(-u) > 1-u$ where $u\in (0,1)$, it is readily seen that $h(u)> 0$ is held for $u\in (0,1)$.

Consequently, by comparing Eq. (\ref{eq.MIM_perfect_stability}) with Eq. (\ref{eq.KL1_perfect_stability}), it is not difficult to see this corollary is true.

\section{Proof of Proposition \ref{prop.proportion_rare_events_MIM}}\label{app.prop.proportion_rare_events_MIM}
By substituting the relationship between $\mathbb{P}$ and $\mathbb{P}_{g_{\theta}}$ described as Eq. (\ref{eq.Pg_Pr}) into Eq. (\ref{eq.L_Doptimal}), it is not difficult to see that
\begin{equation}
\begin{aligned}
    & L_{\text{\rm MIM}}(D=D^{*}_{\text{\rm MIM}},G)\\
    & = \sqrt{{\rm e}} \bigg\{
     \mathbb{E}_{{\bm x}\sim \mathbb{P}}\bigg[\bigg(\frac{P({\bm x})}{P_{g_{\theta}}({\bm x})}\bigg)^{-\frac{1}{2}}\bigg]
    +\mathbb{E}_{{\bm x}\sim \mathbb{P}_{g_{\theta}}}\bigg[\bigg(\frac{P_{g_{\theta}}({\bm x})}{P({\bm x})}\bigg)^{-\frac{1}{2}}\bigg]
    \bigg\}\\
    & = \sqrt{{\rm e}}
    \int_{\mathcal{X}} \big\{ P({\bm x}) (1+ \varepsilon P^{\gamma-1}({\bm x})\rho({\bm x}))^{\frac{1}{2}} 
    + ( P({\bm x}) + \varepsilon P^{\gamma}({\bm x})\rho({\bm x}) )
    (1+\varepsilon P^{\gamma-1}({\bm x})\rho({\bm x}) )^{-\frac{1}{2}}
    \big\} {\rm d}{\bm x} \\
    & = 2\sqrt{{\rm e}} \int_{\mathcal{X}} P({\bm x})
    (1+ \varepsilon P^{\gamma-1}({\bm x})\rho({\bm x}))^{\frac{1}{2}} {\rm d}{\bm x} \\
    & \overset{(a)}{=} 2\sqrt{{\rm e}} \int_{\mathcal{X}} P({\bm x})
    \big\{ 1+ \frac{1}{2}\varepsilon P^{\gamma-1}({\bm x})\rho({\bm x}) 
    - \frac{1}{8} \varepsilon^2 P^{2\gamma-2}({\bm x})\rho^2({\bm x}) + o(\varepsilon^2) \big\} {\rm d}{\bm x},\\
    & = 2\sqrt{{\rm e}} \bigg\{
     1-\frac{1}{8} \varepsilon^2 \int_{\mathcal{X}} P^{2\gamma-1}({\bm x})\rho^2({\bm x}) {\rm d}{\bm x}
     + o(\varepsilon^2)  \bigg\},
\end{aligned}
\end{equation}
where the equality $(a)$ is derived from Taylor's Series Expansion Theorem.

Based on the above discussion, the part of $L_{\text{MIM}}(D=D^{*}_{\text{MIM}},G)$ influenced by small probability events is given by
\begin{equation}
\begin{aligned}
    & \Lambda_{L_{\text{MIM}}(D=D^{*}_{\text{MIM}},G)} \\
    & = \sqrt{{\rm e}} \bigg\{
     \int_{\Omega_{\text{\rm small}}} P({\bm x}) \bigg(\frac{P({\bm x})}{P_{g_{\theta}}({\bm x})}\bigg)^{-\frac{1}{2}} {\rm d}{\bm x} 
    + \int_{\Omega_{\text{\rm small}}} P_{g_{\theta}}({\bm x}) \bigg(\frac{P_{g_{\theta}}({\bm x})}{P({\bm x})}\bigg)^{-\frac{1}{2}} {\rm d}{\bm x}
    \bigg\}\\
    & = 2\sqrt{{\rm e}} \int_{\Omega_{\text{\rm small}}} P({\bm x})
    (1+ \varepsilon P^{\gamma-1}({\bm x})\rho({\bm x}))^{\frac{1}{2}} {\rm d}{\bm x} \\
    & = 2\sqrt{{\rm e}} \int_{\Omega_{\text{\rm small}}} P({\bm x})
    \big\{ 1+ \frac{1}{2}\varepsilon P^{\gamma-1}({\bm x})\rho({\bm x})
    - \frac{1}{8} \varepsilon^2 P^{2\gamma-2}({\bm x})\rho^2({\bm x}) + o(\varepsilon^2) \big\} {\rm d}{\bm x},
\end{aligned}
\end{equation}
for which, the proportion of small probability events in the value of
$L_{\text{MIM}}(D=D^{*}_{\text{MIM}},G)$ is given by
\begin{equation}\label{eq.rare_MIM}
\begin{aligned}
    \Upsilon_{\Omega_{\text{\rm small}}}^{\text {\rm MIM}} 
    & = \frac{\Lambda_{L_{\text{MIM}}(D=D^{*}_{\text{MIM}},G)}}{L_{\text{MIM}}(D=D^{*}_{\text{MIM}},G)}\\
    & = \frac{
    \int_{\Omega_{\text{\rm small}}}
    \{ \frac{2P({\bm x}) +\varepsilon P^{\gamma}({\bm x})\rho({\bm x})}{2}
    - \frac{\varepsilon^2 P^{2\gamma-1}({\bm x})\rho^2({\bm x})}{8}   \} {\rm d}{\bm x}
    + o(\varepsilon^2)
    }
    {1-\frac{1}{8}\varepsilon^2 \int_{\mathcal{X}} P^{2\gamma-1}({\bm x})\rho^2({\bm x})
    {\rm d}{\bm x} + o(\varepsilon^2)
    }\\
    & \approx \frac{\frac{ \Phi\{{\bm x}\in \Omega_{\text{\rm small}}\}
    + \Phi_{g_{\theta}}\{{\bm x}\in \Omega_{\text{\rm small}}\}}{2}
    -\frac{\varepsilon^2}{8} \int_{\Omega_{\text{\rm small}}} P^{2\gamma-1}({\bm x})\rho^2({\bm x})
    {\rm d}{\bm x} }
    { 1-\frac{\varepsilon^2}{8} \int_{\mathcal{X}} P^{2\gamma-1}({\bm x})\rho^2({\bm x})
    {\rm d}{\bm x} },
\end{aligned}
\end{equation}
where the notations are the same as those in Eq. (\ref{eq.R_small_MIM}).

In addition, with regard to the condition $\mathbb{P}_{g_{\theta}}= \{ p+\varepsilon p^{\gamma} \rho_p, 1-p-\varepsilon p^{\gamma}\rho_p \} = \{q,1-q\}$ ($q<\frac{1}{2}$),
we pay attention to the small probability events which can be reflected into the probability element $p$ rather than $(1-p)$ $(0<p\ll\frac{1}{2})$.
In this case, similar to Eq. (\ref{eq.rare_MIM}), the proportion of small probability events in the value of objective function is given by
\begin{equation}\label{eq.binary_rare_MIM}
\begin{aligned}
    \Upsilon_{\Omega_{\text{\rm small}}}^{\text {\rm MIM}} 
    & = \frac{
    2\sqrt{{\rm e}} \big\{ p (1+\varepsilon p^{\gamma-1}\rho_p)^{\frac{1}{2}} \big\}
    }{
    2\sqrt{{\rm e}} \big\{
    p (1+\varepsilon p^{\gamma-1}\rho_p)^{\frac{1}{2}}+ (1-p)\bigg(1-\frac{\varepsilon p^{\gamma}\rho_p}{1-p}\bigg)^{\frac{1}{2}}
    \big\}
    }\\
    & = \frac{
    p+\frac{1}{2} \varepsilon p^{\gamma}\rho_p-\frac{1}{8}\varepsilon^2p^{2\gamma-1}\rho_p^2
    + o(\varepsilon^2)
    }
    { T_{\text{MIM}}
    }\\
    & =  \frac{\frac{p+q}{2}-\frac{1}{8}\varepsilon^2p^{2\gamma-1}\rho_p^2 + o(\varepsilon^2) }
    { 1-\frac{1}{8}\varepsilon^2 \frac{p^{2\gamma-1}\rho_p^2}{1-p} + o(\varepsilon^2)} 
    \approx \frac{\frac{p+q}{2}-\frac{1}{8}\varepsilon^2p^{2\gamma-1}\rho_p^2}
    { 1-\frac{1}{8} \varepsilon^2 \frac{ p^{2\gamma-1}\rho_p^2}{1-p}}\text{,}
\end{aligned}
\end{equation}
where $T_{\text{MIM}}=\{ p+\frac{1}{2} \varepsilon p^{\gamma}\rho_p - \frac{1}{8}\varepsilon^2p^{2\gamma-1}\rho_p^2+ o(\varepsilon^2) \}$
$+ \{ (1-p)-\frac{1}{2} \varepsilon p^{\gamma}\rho_p-\frac{1}{8}\varepsilon^2 \frac{p^{2\gamma}\rho_p^2}{1-p} + o(\varepsilon^2) \}$.

Therefore, this proposition is verified completely.

\section{Proof of Corollary \ref{cor.proportion_rare_events}}\label{app.cor.proportion_rare_events}
Similar to Proposition \ref{prop.proportion_rare_events_MIM}, the small probability event analysis for the original GAN is also available. We have
\begin{equation}
\begin{aligned}
    & L_{\text{KL}}(D=D^{*}_{\text{KL}},G)\\
    & = \mathbb{E}_{{\bm x}\sim \mathbb{P}}
    \bigg[\ln \frac{{P}({\bm x})}{{P}({\bm x})+ P_{g_{\theta}}({\bm x})}\bigg]
    +\mathbb{E}_{{\bm x}\sim \mathbb{P}_{g_{\theta}}}
    \bigg[\ln \frac{ P_{g_{\theta}}({\bm x})}{{P}({\bm x})+ P_{g_{\theta}}({\bm x})}\bigg]\\
    & = \int_{\mathcal{X}} \big\{
    - P({\bm x})\ln (2+\varepsilon P^{\gamma-1}({\bm x}) \rho({\bm x}))
    + (P({\bm x})+ \varepsilon P^{\gamma}(\bm x)\rho(\bm x))
    \ln ( 1-\frac{1}{2+\varepsilon P^{\gamma-1}({\bm x})\rho({\bm x})} )
    \big\} {\rm d}{\bm x} \\
    & = \int_{\mathcal{X}} \big\{
    P({\bm x}) \big[ -\ln2 -\frac{\varepsilon P^{\gamma-1}({\bm x})\rho({\bm x})}{2}
    +\frac{\varepsilon^2 P^{2\gamma-2}({\bm x})\rho^2({\bm x})}{8}
     + o(\varepsilon) \big] \\
    & \quad + (P({\bm x}) + \varepsilon P^{\gamma}\rho({\bm x}))
    \big[ -\ln2 + \frac{\varepsilon P^{\gamma-1}({\bm x})\rho({\bm x})}{2} 
     - \frac{3 \varepsilon^2 P^{2\gamma-2}({\bm x})\rho^2({\bm x})}{8} + o(\varepsilon^2) \big]
    \big\} {\rm d}{\bm x}\\
    & = -2\ln2
    + \frac{\varepsilon^2}{4} \int_{\mathcal{X}} P^{2\gamma-1}({\bm x})\rho^2({\bm x}) {\rm d}{\bm x}
    + o(\varepsilon^2),
\end{aligned}
\end{equation}
in which the part influenced by small probability events is described as
\begin{equation}
\begin{aligned}
    & \Lambda_{L_{\text{KL}}(D=D^{*}_{\text{KL}},G)} \\
    & = \int_{\Omega_{\text{small}}} P({\bm x})
    \bigg[\ln \frac{{P}({\bm x})}{{P}({\bm x})+ P_{g_{\theta}}({\bm x})}\bigg] {\rm d}{\bm x}
    + \int_{\Omega_{\text{small}}} P_{g_{\theta}}({\bm x})
    \bigg[\ln \frac{ P_{g_{\theta}}({\bm x})}{{P}({\bm x})+ P_{g_{\theta}}({\bm x})}\bigg]
    {\rm d}{\bm x} \\
    & = \int_{\Omega_{\text{small}}} \big\{
    -P({\bm x})\ln (2+\varepsilon P^{\gamma-1}({\bm x}) \rho({\bm x})) 
    + (P({\bm x})+ \varepsilon P^{\gamma}(\bm x)\rho(\bm x))
    \ln ( 1-\frac{1}{2+\varepsilon P^{\gamma-1}({\bm x})\rho({\bm x})} )
    \big\} {\rm d}{\bm x} \\
    & = \int_{\Omega_{\text{small}}} \big\{ -2\ln2 \cdot P({\bm x})
    - \ln2 \cdot \varepsilon P^{\gamma}({\bm x})\rho({\bm x}) 
    + \frac{\varepsilon^2 P^{2\gamma-1}({\bm x})\rho^2({\bm x})}{4} \big\} {\rm d}{\bm x}
    + o(\varepsilon^2).
\end{aligned}
\end{equation}
Thus, as for the original GAN, the proportion influenced by small probability events in $L_{\text{KL}}(D=D^{*}_{\text{KL}},G)$ is given by
\begin{equation}\label{eq.rare_KL}
\begin{aligned}
    & \Upsilon_{\Omega_{\text{\rm small}}}^{\text {\rm KL}} \\
    &= \frac{\Lambda_{L_{\text{KL}}(D=D^{*}_{\text{KL}},G)}}{L_{\text{KL}}(D=D^{*}_{\text{KL}},G)}\\
    & 
    =  \frac{
     \int_{\Omega_{\text{small}}} \big\{
     \frac{2P({\bm x}) + \varepsilon P^{\gamma}({\bm x})\rho({\bm x})}{2}
     - \frac{\varepsilon^2 P^{2\gamma-1}({\bm x})\rho^2({\bm x})}{8\ln2} \big\} {\rm d}{\bm x}
     + o(\varepsilon^2)
    }
    {
    1
    - \frac{\varepsilon^2}{8\ln2} \int_{\mathcal{X}} P^{2\gamma-1}({\bm x})\rho^2({\bm x}) {\rm d}{\bm x}
    + o(\varepsilon^2)
    } \\
    & \approx \frac{\frac{ \Phi\{{\bm x}\in \Omega_{\text{\rm small}}\}
    + \Phi_{g_{\theta}}\{{\bm x}\in \Omega_{\text{\rm small}}\}}{2}
    -\frac{\varepsilon^2}{8\ln2} \int_{\Omega_{\text{\rm small}}} P^{2\gamma-1}({\bm x})\rho^2({\bm x})
    {\rm d}{\bm x} }
    { 1-\frac{\varepsilon^2}{8\ln2} \int_{\mathcal{X}} P^{2\gamma-1}({\bm x})\rho^2({\bm x})
    {\rm d}{\bm x} },
\end{aligned}
\end{equation}
where the notations are the same as those in Eq. (\ref{eq.R_small_MIM}).

When it is satisfied that $\int_{\Omega_{\text{\rm small}}} P({\bm x}) {\rm d}{\bm x} = O(\varepsilon)$,
it is not difficult to see that
\begin{equation}\label{eq.rare_MIM_minus_rare_KL}
\begin{aligned}
    & \varepsilon^2 \cdot \int_{\Omega_{\text{\rm small}}} P^{2\gamma-1}({\bm x})\rho^2({\bm x}){\rm d}{\bm x}
    -
    \bigg\{
    \varepsilon^2 \cdot \int_{\mathcal{X}} P^{2\gamma-1}({\bm x})\rho^2({\bm x}){\rm d}{\bm x}
    \bigg\} 
     \cdot
    \int_{\Omega_{\text{small}}} \bigg\{ P({\bm x}) 
    + \frac{\varepsilon P^{\gamma}({\bm x})\rho({\bm x})}{2} \bigg\} {\rm d}{\bm x} \\
    & =  O(\varepsilon^2) - O(\varepsilon^3) \\
    & \ge 0,
\end{aligned}
\end{equation}
from which it is readily seen that the term of right-hand side in Eq. (\ref{eq.rare_MIM}) is not less than that in Eq. (\ref{eq.rare_KL}), which verifies Eq. (\ref{eq.Lambda_MIM_KL}).

When the assumption in Proposition \ref{prop.proportion_rare_events_MIM} is satisfied for LSGAN, it is not difficult to obtain that
\begin{equation}
\begin{aligned}
    & L_{\text{LS}}(D=D^{*}_{\text{LS}},G)\\
    & = \frac{1}{2}\mathbb{E}_{{\bm x}\sim \mathbb{P}}
    \bigg[\bigg(\frac{{P}({\bm x})}{{P}({\bm x})+ P_{g_{\theta}}({\bm x})} -1 \bigg)^2\bigg]
     +\frac{1}{2} \mathbb{E}_{{\bm x}\sim \mathbb{P}_{g_{\theta}}}
    \bigg[\bigg( \frac{{P}({\bm x})}{{P}({\bm x})+ P_{g_{\theta}}({\bm x})}\bigg)^2\bigg]\\
    & = \frac{1}{2} \int_{\mathcal{X}} \bigg\{
    P({\bm x}) \bigg( 1-\frac{1}{2+\varepsilon P^{\gamma-1}({\bm x})\rho({\bm x})} \bigg)^{2} 
    + (P({\bm x})+\varepsilon P^{\gamma}({\bm x})\rho({\bm x}))
    \bigg( \frac{1}{2+\varepsilon P^{\gamma-1}({\bm x})\rho({\bm x})} \bigg)^{2} \bigg\}
    {\rm d}{\bm x}\\
    & = \frac{1}{4}
    - \frac{\varepsilon^2}{16} \int_{\mathcal{X}} P^{2\gamma-1}({\bm x})\rho^2({\bm x}) {\rm d}{\bm x}
    + o(\varepsilon^2),
\end{aligned}
\end{equation}
as well as the part influenced by small probability events in $L_{\text{LS}}(D=D^{*}_{\text{LS}},G)$ is given by
\begin{equation}
\begin{aligned}
    & \Lambda_{L_{\text{LS}}(D=D^{*}_{\text{LS}},G)} \\
    & = \frac{1}{2} \int_{\Omega_{\text{small}}} P({\bm x})
    \bigg( \frac{P({\bm x})}{P({\bm x})+ P_{g_{\theta}}({\bm x})}-1 \bigg)^2 {\rm d}{\bm x}
     +\frac{1}{2} \int_{\Omega_{\text{small}}} P_{g_{\theta}}({\bm x})
    \bigg( \frac{ P({\bm x})}{{P}({\bm x})+ P_{g_{\theta}}({\bm x})} \bigg)^2 {\rm d}{\bm x}\\
    & = \frac{1}{2} \int_{\Omega_{\text{small}}} \bigg\{
    P({\bm x}) \bigg( 1-\frac{1}{2+\varepsilon P^{\gamma-1}({\bm x})\rho({\bm x})} \bigg)^{2} 
     + (P({\bm x})+\varepsilon P^{\gamma}({\bm x})\rho({\bm x}))
    \bigg( \frac{1}{2+\varepsilon P^{\gamma-1}({\bm x})\rho({\bm x})} \bigg)^{2} \bigg\}
    {\rm d}{\bm x}\\
    & = \frac{1}{4} \int_{\Omega_{\text{small}}} \big\{ P({\bm x})
    + \frac{\varepsilon P^{\gamma}({\bm x}) \rho({\bm x})}{2} 
    - \frac{\varepsilon^2 P^{2\gamma-1}({\bm x}) \rho^2({\bm x})}{4} \big\}
    {\rm d}{\bm x} + o(\varepsilon^2), \\
\end{aligned}
\end{equation}
from which, the proportion of small probability events in $L_{\text{LS}}(D=D^{*}_{\text{LS}},G)$ is derived as
\begin{equation}\label{eq.rare_LS}
\begin{aligned}
    \Upsilon_{\Omega_{\text{\rm small}}}^{\text {\rm LS}} 
    &= \frac{\Lambda_{L_{\text{LS}}(D=D^{*}_{\text{LS}},G)}}{L_{\text{LS}}(D=D^{*}_{\text{LS}},G)}\\
    & =  \frac{
     \int_{\Omega_{\text{small}}} \big\{
     \frac{2P({\bm x}) + \varepsilon P^{\gamma}({\bm x})\rho({\bm x})}{2}
     - \frac{\varepsilon^2 P^{2\gamma-1}({\bm x})\rho^2({\bm x})}{4} \big\} {\rm d}{\bm x}
     + o(\varepsilon^2)
    }
    {
    1
    - \frac{\varepsilon^2}{4} \int_{\mathcal{X}} P^{2\gamma-1}({\bm x})\rho^2({\bm x}) {\rm d}{\bm x}
    + o(\varepsilon^2)
    } \\
    & \approx \frac{\frac{ \Phi\{{\bm x}\in \Omega_{\text{\rm small}}\}
    + \Phi_{g_{\theta}}\{{\bm x}\in \Omega_{\text{\rm small}}\}}{2}
    -\frac{\varepsilon^2}{4} \int_{\Omega_{\text{\rm small}}} P^{2\gamma-1}({\bm x})\rho^2({\bm x})
    {\rm d}{\bm x} }
    { 1-\frac{\varepsilon^2}{4} \int_{\mathcal{X}} P^{2\gamma-1}({\bm x})\rho^2({\bm x})
    {\rm d}{\bm x} }.
\end{aligned}
\end{equation}
Comparing Eq. (\ref{eq.rare_MIM}) with Eq. (\ref{eq.rare_LS}), we hold Eq. (\ref{eq.Lambda_MIM_LS})
which results from Eq. (\ref{eq.rare_MIM_minus_rare_KL}), similar to the comparison for $\Upsilon_{\Omega_{\text{\rm small}}}^{\text {\rm MIM}}$ and $\Upsilon_{\Omega_{\text{\rm small}}}^{\text {\rm KL}}$.

Therefore, the proof of this corollary is completed.

\section{Proof of Proposition \ref{prop.R_large_probability}}\label{app.prop.R_large_probability}

If there exist two classes for the real data, namely the large probability event set $\Omega_{\text{large}}$ as well as the small probability event set $\Omega_{\text{small}}$,
the proportion of large probability events in the objective function (described as Eq. (\ref{eq.LMIM})) is given by
\begin{equation}\label{eq.R_large}
\begin{aligned}
    \Upsilon_{\Omega_{\text{\rm large}}}^{\text {\rm MIM}} 
    & = \frac{\int_{\Omega_{\text{large}}}[ P({\bm x})\exp(1-D({\bm x})) +  P_{g_{\theta}}({\bm x})\exp(D({\bm x}))] {\rm d}{\bm x}}
    {\int_{\Omega_{\text{large}}+\Omega_{\text{small}}}[ P({\bm x})\exp(1-D({\bm x})) +  P_{g_{\theta}}({\bm x})\exp(D({\bm x}))] {\rm d}{\bm x}}\\
    & \overset{(b)}{=}
    \frac{
     \int_{\Omega_{\text{large}}}
     \big\{ P({\bm x})\big(\frac{P({\bm x})}{P_{g_{\theta}}({\bm x})}\big)^{-\frac{1}{2}}
    + P_{g_{\theta}}({\bm x})
    \big(\frac{P_{g_{\theta}}({\bm x})}{P({\bm x})}\big)^{-\frac{1}{2}} \big\}
    {\rm d}{\bm x}
    }
    {
     \int_{\Omega_{\text{large}}+\Omega_{\text{small}}}
     \big\{ P({\bm x})\big(\frac{P({\bm x})}{P_{g_{\theta}}({\bm x})}\big)^{-\frac{1}{2}}
    + P_{g_{\theta}}({\bm x})
    \big(\frac{P_{g_{\theta}}({\bm x})}{P({\bm x})}\big)^{-\frac{1}{2}} \big\}
    {\rm d}{\bm x}
    }\\
    & =
    \frac{
     \int_{\Omega_{\text{large}}}
     [{P({\bm x})}{P_{g_{\theta}}({\bm x})}]^{\frac{1}{2}} {\rm d}{\bm x}
    }
    {
     \int_{\Omega_{\text{large}}+\Omega_{\text{small}}}
     [{P({\bm x})}{P_{g_{\theta}}({\bm x})}]^{\frac{1}{2}} {\rm d}{\bm x}
    }\text{,}
\end{aligned}
\end{equation}
in which the equality $(b)$ is derived by replacing the discriminator $D$ with $D_{\text{MIM}}^*$ mentioned in Eq. (\ref{eq.D*}).

Furthermore, when the relationship between real distribution and generative distribution is described as Eq. (\ref{eq.Pg_Pr}), we have the the proportion of large probability events in the value of objective function as
\begin{equation}
\begin{aligned}
& \Upsilon_{\Omega_{\text{\rm large}}}^{\text {\rm MIM}}\\
& =
\frac{
\int_{\Omega_{\text{large}}}
\{P({\bm x})[ P({\bm x}) + \varepsilon P^{\gamma}({\bm x}) \rho({\bm x}) ]\}^{\frac{1}{2}} {\rm d}{\bm x}
}
{
\int_{\Omega_{\text{large}}+\Omega_{\text{small}}}
\{P({\bm x})[ P({\bm x}) + \varepsilon P^{\gamma}({\bm x}) \rho({\bm x}) ]\}^{\frac{1}{2}} {\rm d}{\bm x}
}\\
& =
\frac{
\int_{\Omega_{\text{large}}}
P({\bm x})[ 1 + \varepsilon P^{\gamma-1}({\bm x}) \rho({\bm x}) ]^{\frac{1}{2}} {\rm d}{\bm x}
}
{
\int_{\Omega_{\text{large}}+\Omega_{\text{small}}}
P({\bm x})[ 1 + \varepsilon P^{\gamma-1}({\bm x}) \rho({\bm x}) ]^{\frac{1}{2}} {\rm d}{\bm x}
}\\
& \overset{(c)}{=}
\frac{
\int_{\Omega_{\text{large}}}
P({\bm x})[ 1 + \frac{1}{2}\varepsilon P^{\gamma-1}({\bm x}) \rho({\bm x}) + o(\varepsilon) ] {\rm d}{\bm x}
}
{
\int_{\Omega_{\text{large}}+\Omega_{\text{small}}}
P({\bm x})[ 1 + \frac{1}{2}\varepsilon P^{\gamma-1}({\bm x}) \rho({\bm x}) + o(\varepsilon) ] {\rm d}{\bm x}
}\\
& =
\frac{ \int_{\Omega_{\text{large}}}P({\bm x}) {\rm d}{\bm x} + O(\varepsilon) }
{
\int_{\Omega_{\text{large}}+\Omega_{\text{small}}} P({\bm x}) {\rm d}{\bm x} + O(\varepsilon)
} 
=
\frac{ \int_{\Omega_{\text{large}}}P({\bm x}) {\rm d}{\bm x} + O(\varepsilon) }
{1+ O(\varepsilon)},
\end{aligned}
\end{equation}
where the notations of $\rho({\bm x})$, $\gamma$ and $\varepsilon$ are the same as those in Eq. (\ref{eq.Pg_Pr}), as well as $O(\cdot)$ and $o(\cdot)$ denote the equivalent infinitesimal and
infinitesimal of higher order, respectively. Moreover, the equality $(c)$ results from Taylor's Series Expansion Theorem.

Furthermore, if the small probability events satisfy $\int_{\Omega_{\text{small}}}{P({\bm x})} {\rm d}{\bm x} < \zeta$ (where the parameter satisfies $\zeta \ll 1$), we have
\begin{equation}
\begin{aligned}
\Upsilon_{\Omega_{\text{\rm large}}}^{\text {\rm MIM}}
& =
\frac{ 1- \int_{\Omega_{\text{small}}}P({\bm x}) {\rm d}{\bm x} + O(\varepsilon) }
{1+ O(\varepsilon)} \\
& \ge
\frac{ 1-\zeta + O(\varepsilon) }
{1+ O(\varepsilon)}
\overset{\varepsilon \to 0}{\longrightarrow}
1-\zeta,
\end{aligned}
\end{equation}
which verifies this proposition.

\ifCLASSOPTIONcaptionsoff
  \newpage
\fi

%
%
%
%
%
%
%


\begin{thebibliography}{1}

\bibitem{An-intelligent}
Y. Liu, M. Xiao, S. Chen, F. Bai, J. Pan, D. Zhang,
``An intelligent edge-chain enabled access control mechanism for IoV,''
\textit{IEEE Internet Things J.},
vol. 8, no. 15,
pp. 12231--12241, Feb. 2021.

\bibitem{DeepEDN}
Y. Ding, G. Wu, D. Chen, N. Zhang, L. Gong, M. Cao and Z. Qin,
``DeepEDN: A deep-learning-based image encryption and decryption network for internet of medical things,''
\textit{IEEE Internet Things J.},
vol. 8, no. 3,
pp. 1504--1518, Feb. 2021.

\bibitem{Recent-GAN-survey}
Z. Pan, W. Yu, X. Yi, A. Khan, F. Yuan and Y. Zheng,
``Recent progress on Generative Adversarial Networks (GANs): A survey,''
\textit{IEEE Access},
vol. 7, pp. 36322--36333, Apr. 2019.

\bibitem{Generative-Adversarial-Nets}
I. J. Goodfellow, J. Pouget-Abadie, M. Mirza, B. Xu, D. Warde-Farley, S. Ozair, A. Courville and Y. Bengio,
``Generative adversarial nets,''
in \textit{Proc. Advances in Neural Information Processing Systems (NeurIPS 2014)},
Montreal, Quebec, Canada,
Dec. 8--13, 2014,
pp. 2672--2680.

\bibitem{Selective-unsupervised}
J. H. Seong and D. H. Seo,
``Selective unsupervised learning-based Wi-Fi fingerprint system using autoencoder and GAN,''
\textit{IEEE Internet Things J.},
vol. 7, no. 3,
pp. 1898--1909, Mar. 2020.

\bibitem{Applying-cross-modality}
C. Xu, Z. Gao, D. Zhang, J. Zhang, L. Xu and S. Li,
``Applying cross-modality data processing for infarction learning in medical internet of things,''
\textit{IEEE Internet Things J.},
early access,
pp. 1--9, 2021.

\bibitem{ESR-GAN}
X. Kang, L. Liu and H. Ma,
``ESR-GAN: Environmental signal reconstruction learning with generative adversarial network,''
\textit{IEEE Internet Things J.},
vol. 8, no. 1,
pp. 636-- 646,
Jan. 1, 2021.

\bibitem{Generating-Videos}
C. Vondrick, H. Pirsiavash and A.Torralba,
``Generating videos with scene dynamics,''
In \textit{Proc. 30th Conference on Neural Information Processing Systems (NeurIPS 2016)},
Barcelona, Spain,
Dec. 5--10, 2016,
pp. 1--9.


\bibitem{Imitating-driver-behavior}
A. Kuefler, J. Morton, T. A. Wheeler and M. J. Kochenderfer,
``Imitating driver behavior with generative adversarial networks,''
in \textit{Proc. IEEE Intelligent Vehicles Symposium (IV 2017)},
Los Angeles, CA, USA,
Jun. 11--14, 2017,
pp. 204--211.

\bibitem{A-novel-semi-supervised}
H. Han , W. Ma, M. Zhou, Q. Guo and A. Abusorrah,
``A novel semi-supervised learning approach to pedestrian reidentification,''
\textit{IEEE Internet Things J.},
vol. 8, no. 4,
pp. 3042--3052,
Feb. 15, 2021.

\bibitem{The-secure-steganography}
Fu Z, Wang F and Cheng X,
``The secure steganography for hiding images via GAN,''
\textit{EURASIP Journal on Image and Video Processing},
vol. 2020, no. 46,
pp. 1--18, Oct. 2020.
\bibitem{fraud-detection}
U. Fiore, A. D. Santis, F. Perla, P. Zanetti and F. Palmieri,
``Using generative adversarial networks for improving classification effectiveness in credit card fraud detection,''
\textit{Inf. Sci.},
vol. 479, no. 10,
pp. 448--455, Apr. 2019.
\bibitem{Big-data-driven}
L. Stojanovic, M. Dinic, N. Stojanovic and A. Stojadinovic,
``Big-data-driven anomaly detection in industry (4.0): An approach and a case study,''
in \textit{Proc. IEEE International Conference on Big Data (Big Data 2016)},
Washington, DC, USA,
Dec. 5--8, 2016, pp. 1647--1652,
\bibitem{ATD}
H. Soleimani and D. J. Miller,
``ATD: Anomalous topic discovery in high dimensional discrete data,''
\textit{IEEE Trans. Knowl. Data Eng.},
vol. 28, no. 9,
pp. 2267--2280, Sep. 2016.
\bibitem{Probabilistic-mismatch}
L. Zhang, X. Li, H. Liu, J. Mei, G. Hu, J. Zhao, Y. Zou, B. Xie and G. Xie,
``Probabilistic-mismatch anomaly detection: Do ones medications match with the diagnoses,''
in \textit{Proc. IEEE International Conference on Data Mining (ICDM 2016)},
Barcelona, Spain,
Dec. 12--15, 2016,
pp. 659–668.

\bibitem{Big_Data_IoT}
J. Granat, J. Batalla, C. Mavromoustakis and G. Mastorakis,
``Big data analytics for event detection in the IoT-multicriteria approach,''
\textit{IEEE Internet Things J.},
vol. 7, no. 5,
pp. 4418--4430, May. 2020.

\bibitem{Two_tier}
M. Gajewski, J. Batalla, A. Levi, C. Togay, C. Mavromoustakis and G. Mastorakis,
``Two-tier anomaly detection based on traffic profiling of the Home Automation system,''
\textit{Computer Network},
vol. 158,
pp. 46--60, Jul. 2019.

\bibitem{6G_Enabled}
G. Han, J. Tu, L. Liu, M. Mart{\'{i}}nez-Garc{\'{i}}a, Y. Peng,
``Anomaly detection based on multidimensional data processing for protecting vital devices in
6G-enabled massive IIoT,''
\textit{IEEE Internet Things J.},
vol. 8, no. 7,
pp. 5219--5229, Apri. 2021.

\bibitem{Cyber_Physical}
P. Araujo-Filho , G. Kaddoum, D.Campelo, A. Santos, D. Mac\v{e}do and C. Zanchettin,
``Intrusion detection for cyber–hysical systems using generative adversarial networks in fog environment,''
\textit{IEEE Internet Things J.},
vol. 8, no. 8,
pp. 624--6256, Apri. 2021.

\bibitem{Distributed_Intrusion}
A. Ferdowsi and W. Saad,
``Generative adversarial networks for distributed intrusion detection in the Internet of Things,''
in \textit{Proc. IEEE Global Communications Conference (GLOBECOM 2019)},
Waikoloa, HI, USA,
Dec. 9--13, 2019.
pp. 1--6.

\bibitem{GAN_Auto_Encoder}
Z. Tian, S. Kushan and G. Mohan,
``Generative adversarial network and auto encoder based anomaly detection in Distributed IoT networks,''
in \textit{Proc. IEEE Global Communications Conference (GLOBECOM 2020)},
Taipei, Taiwan
Dec. 7--11, 2020.
pp. 1--7.

\bibitem{Make-the-rocket}
Y. Feng, Z. Liu, J. Chen, H. Lv, J. Wang and J. Yuan,
``Make the rocket intelligent at IoT edge: Stepwise GAN for anomaly detection of LRE with multi-source fusion,''
\textit{IEEE Internet Things J.},
vol. 9, no. 4,
pp. 3135--3149, Feb. 2022.

\bibitem{CHRIST}
M. Garmaroodi, F. Farivar, M. Haghighi, M. Shoorehdeli and A. Jolfaei,
``Detection of anomalies in industrial IoT systems by data mining: Study of CHRIST osmotron water purification system,''
\textit{IEEE Internet Things J.},
vol. 8, no. 13,
pp. 10280--10287, Jul. 2021.

\bibitem{Monotone}
T. Dang, D. Le, T. Nguyen, M. Kim and H. Choo,
``Monotone split and conquer for anomaly detection in iot sensory data,''
\textit{IEEE Internet Things J.},
vol. 8, no. 20,
pp. 15468--15485, Oct. 2021.

\bibitem{Wireless_Commun}
J. Jiang, G. Han, L. liu, L. Shu and M. Guizani,
``Outlier detection approaches based on machine learning in the Internet-of-Things,''
\textit{IEEE Wireless Commun.},
vol. 27, no. 3,
pp. 53--59, Jun. 2020.

\bibitem{feature-based}
D. Wu, H. Shi, H. Wang, R. Wang and H. Fang,
``A feature-based learning system for Internet of Things applications,''
\textit{IEEE Internet Things J.},
vol. 6, no. 2,
pp. 1928--1937, Jun. 2020.



\bibitem{Isolation-forest}
F. T. Liu, K. M. Ting, and Z. H. Zhou,
``Isolation forest,''
in \textit{Proc. IEEE International Conference on Data Mining (ICDM 2008)},
Pisa, Italy,
Dec. 15--19, 2008,
pp. 413--422.
\bibitem{Isolation-based}
F. T. Liu, K. M. Ting, and Z. H. Zhou,
``Isolation-based anomaly detection,''
\textit{ACM Transactions on Knowledge Discovery from Data (TKDD)},
vol. 6, no.1,
pp. 1--39, Mar. 2012.

\bibitem{Estimating-the-support}
B. Schlkopf, J. C. Platt, J. Shawe-Taylor, A. J. Smola, and R. C. Williamson,
``Estimating the support of a high-dimensional distribution,''
\textit{Neural Comput.},
vol. 13, no. 7,
pp. 1443--1471, Jul. 2001.
\bibitem{High-dimensional-and-large-scale}
S. M. Erfani, S. Rajasegarar, S. Karunasekera, and C. Leckie,
``High-dimensional and large-scale anomaly detection using a linear one-class SVM with deep learning,''
\textit{Pattern Recognit.},
vol. 58,
pp. 121--134, Feb. 2016.


\bibitem{exemplar-based-GMM}
X. Yang, L. J. Latecki and D. Pokrajac,
``Outlier detection with globally optimal exemplar-based GMM,''
in \textit{Proc. SIAM International Conference on Data Mining (SDM 2009)},
Sparks, Nevada, USA,
Apr.30--May. 2, 2009,
pp. 145--154.

\bibitem{Robustmodel-based}
K. Yu, X. Dang, H. Bart and Y. Chen,
``Robustmodel-based learning via Spatial-EM algorithm,''
\textit{IEEE Trans. Knowl. Data Eng.},
vol. 27, no. 6, pp. 1670--1682, Jun. 2015.

\bibitem{Deep-autoencoding-Gaussian}
B. Zong, Q. Song, M. R. Min, W. Cheng, C. Lumezanu, D. Cho and H. Chen,
``Deep autoencoding Gaussian mixture model for unsupervised anomaly detection,''
in \textit{Proc. International Conference on Learning Representations (ICLR 2018)},
Vancouver, BC, Canada,
Apr. 30--May. 3, 2018.

\bibitem{LOF}
M. M. Breunig,
``LOF: Identifying density-based local outliers,''
In \textit{Proc. ACM SIGMOD International Conference on Management Of Data (MOD 2000)},
Dallas, Texas, USA,
May. 16--18, 2000,
pp. 93--104.
\bibitem{Fast-memory}
M. Salehi, C. Leckie, J. C. Bezdek, T. Vaithianathan and X. Zhang,
``Fast memory efficient local outlier detection in data streams,''
\textit{IEEE Trans. Knowl. Data Eng.},
vol. 28, no. 12,
pp. 3246--3260, Dec. 2016.

\bibitem{A-near-linear-time}
N. Pham and R. Pagh,
``A near-linear time approximation algorithm for angle-based outlier detection in high-dimensional data,''
in \textit{Proc. ACM SIGKDD International Conference on Knowledge Discovery and Data Mining (KDD 2012)},
Beijing, China,
Aug. 12--16, 2012,
pp. 877-- 885.

\bibitem{Outlier-detection}
J.Hardin and D.M. Rocke,
``Outlier detection in the multiple cluster setting using the minimum covariance determinant estimator,''
\textit{Comput. Statist. Data Anal.},
vol. 44, no. 4,
pp. 625--638, Jan. 2004.

\bibitem{Direct-robust-matrix}
L. Xiong, X. Chen and J. Schneider,
``Direct robust matrix factorization for anomaly detection,''
in \textit{Proc. IEEE International Conference on Data Mining (ICDM 2012)},
Brussels, Belgium,
Dec. 10--13, 2012, pp. 844--853.
\bibitem{Fast-matrix-factorization}
X. He, H. Zhang, M. Y. Kan and T. S. Chua,
``Fast matrix factorization for online recommendation with implicit feedback,''
in \textit{Proc. ACM Special Interest Group on Information Retrieval (SIGIR 2016)},
Pisa, Italy,
Jul. 17--21, 2016, pp. 549--558.
\bibitem{robust-deep-autoencoders}
C. Zhou and R. C. Paffenroth,
``Anomaly detection with robust deep autoencoders,''
in \textit{Proc. ACM SIGKDD Conference on Knowledge Discovery and Data Mining (KDD 2017)},
Halifax, Nova Scotia, Canada,
Aug. 13--17, 2017, pp. 665--674.
\bibitem{autoencoder-ensembles}
J. Chen, S. Sathe, C. Aggarwal and D. Turaga,
``Outlier detection with autoencoder ensembles,''
in \textit{Proc. SIAM International Conference on Data Mining (SDM 2017)},
Houston, Texas, USA,
Apr. 27--29, 2017, pp. 90--98.


\bibitem{K-means-Clustering-PCA}
C. Ding and X. He,
``K-Means clustering via principal component analysis,''
in \textit{Proc. International Conference on Machine Learning (ICML 2004)},
Banff, Canada,
July 4--8, 2004,
pp. 1--9.
\bibitem{Dimensionality-Reduction}
A. Jamal, A. Handayani, A. A. Septiandri, E. Ripmiatin and Y. Effendi,
``Dimensionality reduction using PCA and K-Means clustering for breast cancer prediction,''
\textit{Lontar Komput. J. Ilm. Teknol. Inf.},
vol. 9, no. 3,
Dec. 2018, pp. 192--201.

\bibitem{A-deep-one-class}
P. Wu, J. Liu and F. Shen,
``A deep one-class neural network for anomalous event detection in complex scenes,''
\textit{IEEE Trans. Neural Netw. Learn. Syst.},
vol. 31, no. 7,
pp. 2609--2622, Jul. 2020.

\bibitem{Anomaly-detection-of-time}
L. Li, J. Yan, H. Wang and Y. Jin,
``Anomaly detection of time series with smoothness-inducing sequential variational auto-encoder,''
\textit{IEEE Trans. Neural Netw. Learn. Syst.},
vol. 32, no. 3,
pp. 1177--1191, Mar. 2021.

\bibitem{Unsupervised-anomaly}
T. Schlegl, P. Seebock, S. M. Waldstein, U. Schmidt-Erfurth and G. Langs,
``Unsupervised anomaly detection with generative adversarial networks to guide marker discovery,''
in \textit{Proc. International Conference on Information Processing in Medical Imaging (IPMI 2017)},
Boone, North Carolina, USA,
Jun. 25--30, 2017,
pp. 146--157.

\bibitem{Adversarially-learned}
H. Zenati, M. Romain, C. S. Foo, B. Lecouat and V. Chandrasekhar,
``Adversarially learned anomaly detection,''
in \textit{Proc. IEEE International Conference on Data Mining (ICDM 2018)},
Singapore,
Nov. 17--20, 2018,
pp. 1--11.
\bibitem{Adversarial-feature}
J. Donahue, P. Krahenbuhl, and T. Darrell,
``Adversarial feature learning,''
in \textit{Proc. International Conference on Learning Representations (ICLR 2017)},
Palais des Congres Neptune, Toulon, France,
Apri. 24--26, 2017,
pp. 1--18.

\bibitem{GANomaly}
S. Akcay, A. Atapourabarghouei, and T. P. Breckon,
``GANomaly: Semi-supervised anomaly detection via adversarial training,''
in \textit{Proc. Asian Conference on Computer Vision (ACCV 2018)},
Perth, Australia,
Dec. 2--6, 2018,
pp 622--637.

\bibitem{f-AnoGAN}
T. Schlegl, P. Seebock, S. M. Waldstein, G. Langs and U. Schmidt- Erfurth,
``f-AnoGAN: Fast unsupervised anomaly detection with generative adversarial networks,''
\textit{Medical Image Analysis},
vol. 54, pp. 30--44, May. 2019.

\bibitem{Telemetry-data-based}
J. Yu, Y. Song, D. Tang, D. Han and J. Dai,
``Telemetry data-based spacecraft anomaly detection with spatial–temporal generative adversarial networks,''
\textit{IEEE Trans. Instrum. Meas.},
vol. 70,
pp. 1--9, Apr. 2021.
\bibitem{Anomaly-monitoring}
H. Kim, J. Park, K. Min and K. Huh,
``Anomaly monitoring framework in lane detection with a generative adversarial network,''
\textit{IEEE Trans. Intell. Transp. Syst.},
vol. 22, no. 3,
pp. 1603--1615,
Mar. 2021.

\bibitem{Weakly-supervised}
T. Jiang, W. Xie, Y. Li, J. Lei and Q. Du,
``Weakly supervised discriminative learning with spectral constrained generative adversarial network for hyperspectral anomaly detection,''
\textit{IEEE Trans. Neural Netw. Learn. Syst.},
early access,
pp. 1--14, 2021.
\bibitem{Semisupervised-spectral}
K. Jiang, W. Xie, Y. Li, J. Lei, G. He and Q. Du,
``Semisupervised spectral learning with generative adversarial network for hyperspectral anomaly detection,''
\textit{IEEE Trans. Geosci. Remote Sens.},
vol. 58, no. 7,
pp. 5224-- 5236,
Jul. 2020.
\bibitem{Unsupervised-pixel-wise}
S. Arisoy, N. M. Nasrabadi and K. Kayabol,
``Unsupervised pixel-wise hyperspectral anomaly detection via autoencoding adversarial networks,''
\textit{IEEE Geosci. Remote Sens. Lett.},
early access,
pp. 1--5, 2021.

\bibitem{RPCA}
T. Yu, X. Wang and A. Shami,
``Recursive principal component analysis-based data outlier detection and sensor data aggregation in IoT systems,''
\textit{IEEE Internet Things J.},
vol. 4, no. 6,
pp. 2207--2216, Dec. 2017.

\bibitem{Fog_empowered}
L. Lyu , J. Jin, S. Rajasegarar, X. He and M. Palaniswami,
``Fog-empowered anomaly detection in IoT using hyperellipsoidal clustering,''
\textit{IEEE Internet Things J.}
vol. 4, no. 5,
pp. 1174-- 1184, Oct. 2017.

\bibitem{Integrated_GAN}
F. Kong, J. Li, B. Jiang, H. Wang and H. Song,
``Integrated generative model for industrial anomaly detection via bi-directional LSTM and attention mechanism,''
\textit{IEEE Trans. Ind. Informat.},
early access,
pp. 1--10, May. 2021.

\bibitem{InfoGAN}
X. Chen, Y. Duan, R. Houthooft, J. Schulman, I. Sutskever and P. Abbeel,
``InfoGAN: Interpretable representation learning by information maximizing generative adversarial nets,''
in \textit{Proc. Advances in Neural Information Processing Systems (NeurIPS 2016)},
Barcelona, Spain, Dec. 5--10, 2016, pp. 2172--2180.

\bibitem{variational-autoencoder-GAN}
A. Larsen, S. Sønderby, H. Larochelle, and O.Winther,
``Autoencoding beyond pixels using a learned similarity metric,''
in \textit{Proc. International Conference on Machine Learning (ICML 2016)},
New York City, NY, USA, Jun. 19--24, 2016, pp. 1558--1566.

\bibitem{Least-squares-generative-adversarial-networks}
X. Mao, Q. Li, H. Xie, R. Y. K. Lau, Z. Wang and S. P. Smolley,
``Least squares generative adversarial networks,''
in \textit{Proc. IEEE International Conference on Computer Vision (ICCV 2017)},
Venice, Italy,
Oct. 22--29, 2017,
pp. 2794--2802.

\bibitem{Wasserstein-GAN}
M. Arjovsky, S. Chintala and L. Bottou,
``Wasserstein GAN,''
https://arxiv.org/abs/1701.07875,
pp. 1--32, 2017.


\bibitem{Improved-Training-of-Wasserstein-GANs}
I. Gulrajani, F. Ahmed, M. Arjovsky, V. Dumoulin and A. Courville,
``Improved training of Wasserstein GANs,''
in \textit{Proc. Advances in Neural Information Processing Systems (NeurIPS 2017)},
Long Beach, CA, USA,
Dec. 3--9, 2017,
pp. 5767--5777.

\bibitem{cGAN}
M. Mirza and S. Osindero,
``Conditional generative adversarial nets,''
https://arxiv.org/abs/1411.1784v1,
pp. 1--7, 2014.

\bibitem{DCGAN}
A. Radford, L. Metz and S. Chintala,
``Unsupervised representation learning with deep convolutional generative adversarial networks,''
https://arxiv.org/abs/1511.06434,
pp. 1--16, 2015.


\bibitem{Bayesian-GAN}
Y. Saatchi and A. Wilson,
``Bayesian GAN,''
in \textit{Proc. Advances in Neural Information Processing Systems (NeurIPS 2017)},
Long Beach, CA, USA, Dec. 3--9, 2017, pp. 3622--3631.

\bibitem{Collapse-GAN}
Z. Zhang, M. Li and J. Yu,
``On the convergence and mode collapse of GAN,''
in \textit{Proc. SIGGRAPH Asia 2018 Technical Briefs},
Tokyo, Japan, Dec. 4--7, 2018, pp. 1--4.

\bibitem{Improved-training}
T. Salimans, I. Goodfellow, W. Zaremba, V. Cheung, A. Radford and X. Chen,
``Improved techniques for training GANs,''
in \textit{Proc. Advances in Neural Information Processing Systems (NeurIPS 2016)},
Barcelona, Spain, Dec. 5--10, 2016, pp. 4844--4852.

\bibitem{CycleGAN}
J. Zhu, T. Park, P. Isola and A. A. Efros,
``Unpaired image-to-image translation using cycle-consistent adversarial networks,''
in \textit{Proc. IEEE International Conference on Computer Vision (ICCV 2017)},
Venice, Italy, Oct. 22--29, 2017, pp. 2242--2251.

\bibitem{PGGAN}
T. Karras, T. Aila, S. Laine and J. Lehtinen,
``Progressive growing of GANs for improved quality, stability, and variation,''
in \textit{Proc. International Conference on Learning Representations (ICLR 2018)},
Vancouver, BC, Canada, Apr. 30--May. 3, 2018, pp. 1--26.

\bibitem{StackGAN}
H. Zhang, T. Xu, H. Li, S. Zhang, X. Wang, X. Huang and D. Metaxas,
``StackGAN: Text to photo-realistic image synthesis with stacked generative adversarial networks,''
in \textit{Proc. IEEE Conference on Computer Vision and Pattern Recognition (CVPR 2017)},
Honolulu, HI, USA, Jul. 21--26, 2017, pp. 5907--5915.

\bibitem{StyleGAN}
T. Karras, S. Laine and T. Aila,
``A style-based generator architecture for generative adversarial networks,''
in \textit{Proc. IEEE Conference on Computer Vision and Pattern Recognition (CVPR 2019)},
Long Beach, CA, USA, Jul. 21--26, 2019, pp. 4401--4410.

\bibitem{BigGAN}
A. Brock, J. Donahue and K. Simonyan,
``Large scale GAN training for high fidelity natural image synthesis,''
in \textit{Proc. International Conference on Learning Representations (ICLR 2019)},
New Orleans, Louisiana, USA, May. 6--9, 2019, pp. 1--35.


\bibitem{Message-importance-measure-and-its-application}
P. Fan, Y. Dong, J. X. Lu, and S. Y. Liu,
``Message importance measure and its application to minority subset detection in big data,''
in \textit{Proc. IEEE Globecom Workshops (GC Wkshps 2016)},
Washington D.C., USA,
Dec. 4--8, 2016,
pp. 1--6.

\bibitem{Amplifying-inter-message-distance}
R. She, S. Y. Liu, and P. Fan,
``Amplifying inter-message distance: On information divergence measures in big data,''
\textit{IEEE Access}, vol. 5,
pp. 24105--24119,
Nov. 2017.

\bibitem{How-GANs}
Y. Hong, U. Hwang, J. Yoo, and S. Yoon,
``How generative adversarial networks and their variants work: An overview,''
\textit{ACM Comput. Surveys},
vol. 52, no. 1, pp. 1--43, Feb. 2019.

\bibitem{GANs-trained}
M. Heusel, H. Ramsauer, T. Unterthiner, B. Nessler, and S. Hochreiter,
``GANs trained by a two time-scale update rule converge to a local nash equilibrium,''
in \textit{Proc. Advances in Neural Information Processing Systems (NeurIPS 2017)},
Long Beach, CA, USA,
Dec. 3--9, 2017,
pp. 6626--6637.

\bibitem{An-information-theoretic-approach}
S. Huang, A. Makur, L. Zheng and G.W. Wornell,
``An information theoretic approach to universal feature selection in high-dimensional inference,''
in \textit{Proc. IEEE International Symposium on Information Theory (ISIT 2017)},
Aachen, Germany,
Jun. 25--30, 2017,
pp. 1336--1340.

\bibitem{Pros-cons}
A. Borji,
``Pros and cons of GAN evaluation measures,''
\textit{Computer Vision and Image Understanding},
vol. 179,
pp. 41--65,
Feb. 2019.

\end{thebibliography}
\end{document}